\documentclass{article}

\usepackage{microtype}
\usepackage{graphicx}
\usepackage{subfig}
\usepackage{tabularx}
\usepackage{booktabs}
\usepackage{multirow}
\usepackage{afterpage}
\usepackage{lipsum}
\usepackage{float}
\usepackage{adjustbox}

\usepackage{hyperref}

\usepackage[accepted]{icml2024}

\usepackage{amsmath}
\usepackage{amssymb}
\usepackage{mathtools}
\usepackage{amsthm}
\usepackage{dsfont}

\usepackage[capitalize,noabbrev]{cleveref}

\theoremstyle{plain}
\newtheorem{theorem}{Theorem}[section]
\newtheorem{proposition}[theorem]{Proposition}
\newtheorem{lemma}[theorem]{Lemma}
\newtheorem{corollary}[theorem]{Corollary}
\theoremstyle{definition}
\newtheorem{definition}[theorem]{Definition}
\newtheorem{assumption}[theorem]{Assumption}
\theoremstyle{remark}

\usepackage[textsize=tiny]{todonotes}

\newcommand{\expect}{\mathbb{E}}
\newcommand{\proba}{\mathbb{P}}
\newcommand{\ind}{\mathds{1}}

\newcommand{\dif}{\mathrm{d}}
\DeclareMathOperator{\Var}{Var}

\newcommand{\pinball}{\mathrm{pinball}}
\newcommand{\indep}{\perp\!\!\!\perp}
\DeclareMathOperator*{\argmin}{arg\,min}

\makeatletter
\newcommand*{\coloneq}{\mathrel{\rlap{%
                       \raisebox{0.3ex}{$\m@th\cdot$}}%
                       \raisebox{-0.3ex}{$\m@th\cdot$}}%
                       =}
\makeatother

\icmltitlerunning{Generalization Bounds for Causal Regression: Insights, Guarantees and Sensitivity Analysis}

\begin{document}

\twocolumn[
\icmltitle{Generalization Bounds for Causal Regression: \\ Insights, Guarantees and Sensitivity Analysis}

\icmlsetsymbol{equal}{*}

\begin{icmlauthorlist}
\icmlauthor{Daniel Csillag}{fgvemap}
\icmlauthor{Claudio José Struchiner}{fgvemap}
\icmlauthor{Guilherme Tegoni Goedert}{fgvemap}
\end{icmlauthorlist}

\icmlaffiliation{fgvemap}{School of Applied Mathematics, Fundação Getúlio Vargas, Rio de Janeiro, Brazil}

\icmlcorrespondingauthor{Daniel Csillag}{daniel.csillag@fgv.br}

\icmlkeywords{Causality, Individual Treatment Effect, Conditional Average Treatment Effect, Quantile Treatment Effect, Potential Outcomes, Generalization Bounds, Statistical Learning Theory, Machine Learning, ICML}

\vskip 0.3in
]

\printAffiliationsAndNotice{}  

\begin{abstract}
    Many algorithms have been recently proposed for causal machine learning.
    Yet, there is little to no theory on their quality, especially considering finite samples.
    In this work, we propose a theory based on generalization bounds that provides such guarantees.
    By introducing a novel change-of-measure inequality, we are able to tightly bound the model loss in terms of the deviation of the treatment propensities over the population, which we show can be empirically limited.
    Our theory is fully rigorous and holds even in the face of hidden confounding and violations of positivity.
    We demonstrate our bounds on semi-synthetic and real data, showcasing their remarkable tightness and practical utility.
\end{abstract}

\section{Introduction}
\label{introduction}

Causal machine learning plays an increasingly important role in many application domains including economics, medicine, education research, and more.
At the core of causal ML is reasoning about potential outcomes.
For instance, using covariates $X$ to make predictions about the potential outcomes $Y^1$ and $Y^0$, which correspond to \emph{what would happen} if the treatment were to be administered ($T=1$) and if the treatment were not to be administered ($T=0$).
This reasoning differs crucially from simply predicting the actual outcome $Y$ given those covariates, which is subject to biases in the training data.

The fundamental problem of causal ML is that the potential outcomes cannot be both observed at the same time.
For instance, whenever $T=1$ we can only observe $Y = Y^1$, and $Y^0$ could vary unpredictably.
To amend this, some strong assumptions are introduced, with the most typical being \emph{ignorability} (that $Y^1, Y^0 \indep T | X$) and \emph{positivity} (that for all $X$ and $a$, $0 < \proba[T=a | X] < 1$).
When these don't hold we are in the field of \emph{sensitivity analysis}.
In this context, we may assume that there are certain unobserved covariates $U$ which must be included to satisfy ignorability.

We consider two key tasks of causal ML: (i) outcome regression, in which we seek to predict the potential outcomes $Y^a$ given the covariates, and (ii) individual treatment effect estimation, in which we seek to predict the treatment effects $Y^1 - Y^0$ given the covariates.

Many methods have been proposed for these tasks.
There are classical approaches based on linear models~\cite{linear-1,linear-2} as well as
more recent approaches based on decision trees~\cite{trees-1,trees-2,trees-3}, neural networks~\cite{nns-1,nns-2,nns-3,nns-4} and even some model-agnostic approaches~\cite{x-learner,r-learner,b-learner}.

However, all these methods still lack a rigorous supporting theory,
with there being fundamental questions that have not been fully answered.
For example, how well can we expect these procedures to extract causal relations from the data?
How many samples are enough?
How do these procedures fare when causal assumptions are violated (e.g., because there are unobserved confounders)?

In this paper, we seek to provide satisfying answers to these questions through the lens of generalization bounds.
By leveraging a change-of-measure inequality based on an $f$-divergence (namely, the Pearson $\chi^2$ divergence) we are able to tightly bound the (unobservable) complete causal loss in terms of the (observable) conditional loss plus some additional highly interpretable terms:

\begin{theorem}[Informal]
    For any (decomposable) loss function and any $\lambda > 0$,
    \[ \underbrace{\expect[\mathrm{Loss}]}_\text{unobservable\dots} \leq \underbrace{\expect[\mathrm{Loss} | T=a]}_\text{observable!} + \lambda \cdot \Delta + \sigma^2/4\lambda \]
    where $\Delta$ is a term that quantifies how far we deviate from a randomized control trial.
\end{theorem}

However, the result above suffers from the fact that $\Delta$ fundamentally depends on unobservable quantities.
But it can -- quite surprisingly -- be empirically upper bounded, leading to a new form of sensitivity analysis:

\begin{theorem}[Informal]
    Let $\nu$ be an (arbitrary) propensity scoring model. Then the following holds:
    \[ \Delta \leq C \cdot \mathrm{BrierScore}(\nu, T) + D \]
    where $D$ quantifies how our reweighed samples deviate from being balanced, $C$ is an universal constant, and $\mathrm{BrierScore}(\nu, T) = \expect[(\nu(X) - T)^2]$.
\end{theorem}

We demonstrate empirically that these bounds are tight and, besides aiding in a theoretical backbone for causal regression, are also useful in practice for model selection.

Our main contributions are:
\begin{itemize}
    \item Novel generalization bounds applicable to many causal regression algorithms, which shed light on the applicability and efficacy of such procedures. Our bounds are general, assumption-light, framework-agnostic (i.e., can be used in conjunction with Rademacher bounds, VC bounds, PAC-Bayes, etc.) and hold even in the lack of ignorability or positivity.
    \item Relaxed versions of our bounds which are entirely empirically boundable with high probability, allowing for practical bounding of the unobservable counterfactual losses.
    \item A change-of-measure inequality based on the Pearson $\chi^2$ $f$-divergence, which is remarkably tight~(Figure~\ref{fig:importance-of-tuning-parameter}) and particularly applicable to causal inference problems.
    \item Our bounds show that we are able to learn to estimate treatment effects while optimizing losses other than the mean squared loss. For example, we can use the mean absolute error for robust estimation, or the quantile loss\footnote{also known as the ``pinball'' loss.} for quantile regression, a feat previously thought impossible in general.
\end{itemize}

\paragraph{Related work}
There has been much work on asymptotic guarantees for causal regression algorithms, e.g., \cite{x-learner,r-learner,b-learner,generalized-random-forest}.
However, besides not being applicable to finite samples, they are often restricted to specific learning classes and/or make restrictive modelling assumptions.
A notable exception to this is \cite{prior-work}, which also establishes some generalization bounds on causal regression. However, their bounds are orders of magnitude looser than ours (see Section~\ref{sec:experiments-semisynthetic}) and are restricted to bounds based on the pseudo-VC-dimension and mean squared error.
Also closely related are generalization bounds for domain adaptation~\cite{domain-adaptation-1,domain-adaptation-2,domain-adaptation-3,domain-adaptation-4,domain-adaptation-5}. Under ignorability, the observed and complete distributions differ by a covariate shift, and causal inference becomes a domain adaptation problem. However, these bounds generally give little intuition when applied in a causal context.

\section{Novel Bounds for Causal Regression}
\label{novel-bounds}

In order to bridge the gap between the complete data distribution and the observed data distribution, we develop a novel change-of-measure inequality based on Pearson's $\chi^2$-divergence:

\begin{definition}[Pearson's $\chi^2$ divergence]\label{def:pearson-chi2}
    Let $\mathcal{H}$ be any arbitrary domain, and denote by $P$ and $Q$ the probability measures over the Borel $\sigma$-field on $\mathcal{H}$. The $\chi^2$ divergence between $Q$ and $P$, denoted $\chi^2(Q \Vert P)$, is given by
    \[ \chi^2(Q \Vert P) \coloneq \expect_P\left[\left(\frac{\dif Q}{\dif P} - 1\right)^2\right]. \]
\end{definition}

\begin{lemma}\label{thm:change-of-measure}
    Let $P$ and $Q$ be as in Definition~\ref{def:pearson-chi2}.
    For any $\lambda > 0$,
    \begin{align*}
        \expect_{Q}[\phi] - E \leq \expect_{P}[\phi] \leq \expect_{Q}[\phi] + E
    \end{align*}
    where
    \[ E \coloneq \lambda \cdot \chi^2(Q \Vert P) + \frac{1}{4\lambda} \Var_P(\phi). \]
    Moreover, the bound is optimized for
    \[ \lambda^\star = \sqrt{\Var_P(\phi) / 4 \chi^2(Q \Vert P)}. \]
\end{lemma}

Note how, if $P = Q$, then the bound in Lemma~\ref{thm:change-of-measure} becomes an equality by taking $\lambda \to \infty$, showing that it is reasonably tight.
This is in contrast to, e.g., the inequalities in~\cite{novel-change-of-measure}, for which this is not the case.

\subsection{Outcome regression}\label{sec:bounds-for-outcome-regression}

For our first setting, we wish to use our covariates $X$ to predict the potential outcome $Y^a$ -- i.e., what would happen to that individual if he received treatment $T=a$. To do so, we consider a sample-reweighted empirical risk minimization problem: that is, we are solving for
\begin{align*}
    h^\star
    &= \argmin_{h \in L^2} \expect[w(X) ( h(X) - Y )^2 | T = a]
    \\ &\approx \argmin_{h \in L^2} \frac{1}{n_{T=a}} \sum_{T_i = a} w(X_i) ( h(X_i) - Y_i )^2.
\end{align*}
The purpose of sample reweighing is to bridge the gap between the observed distribution (conditional on $T=a$) and the complete distribution (unconditional on $T=a$).
In particular, the idea is that $Y^a | X$ is fixed and that reweighing by $w(X)$ changes the distribution over $X$ from $P_{X | T=a}$ to some $P_{\widetilde{X} | T=a} \approx P_X$, with $\dif P_{\widetilde{X}|T=a}/\dif P_{X|T=a} = w(X)$.

This procedure forms the basis for many notable algorithms and is a fundamental building block of causal meta-learners~\cite{x-learner}, one of the main objects of study in this work. The main issue to tackle is quantifying the gap between the observed distribution $Y,X | T=a$ and the complete distribution $Y^a,X$. To this end, we use Lemma~\ref{thm:change-of-measure}:
\begin{align*}
    &\expect[(Y^a - h(X))^2]
    \leq \expect[w(X) (Y - h(X))^2 | T=a] \\
    &+ \lambda \underbrace{\chi^2(P_{Y, \widetilde{X} | T=a} \Vert P_{Y^a, X})}_{\Delta_{T=a}} + \frac{1}{4\lambda} \underbrace{\Var[(Y^a - h(X))^2]}_{\sigma^2_{T=a}};
\end{align*}
We can now leverage the causal nature of our distribution shift in order to precisely quantify this bound. In particular, let us work out the $\chi^2$ term.
Under SUTVA and ignorability with regards to $X$ and $U$ and a simple application of Bayes' rule,\footnote{SUTVA asserts that data is i.i.d. and that $P_{Y^a | T=a} = P_{Y | T=a}$. Ignorability with respect to $X$ and $U$ states that $Y^1, Y^0 \indep T | X, U$.} we have that
\begin{align*}
    &\frac{\dif P_{Y,\widetilde{X},U|T=a}}{\dif P_{Y^a,X,U}}
    = w(X) \frac{\dif P_{Y,X,U|T=a}}{\dif P_{Y^a,X,U}}
    \\ &\quad = w(X) \frac{\dif P_{X,U|T=a}}{\dif P_{X,U}}
    = w(X) \frac{\proba[T=a | X, U]}{\proba[T=a]},
\end{align*}
and so
\begin{align*}
    \Delta_{T=a} =& \chi^2(P_{Y, \widetilde{X} | T=a} \Vert P_{Y^a, X})
    \\ =& \expect\left[\left(w(X) \frac{\proba[T=a | X, U]}{\proba[T=a]} - 1\right)^2\right],
\end{align*}
from which we derive our first major result:
\begin{theorem}[Upper bound on outcome regression loss in expectation]\label{thm:upper-bound-main-theoric-outcome}
    For any $\lambda > 0$, loss function $\ell(\cdot, \cdot)$ and nonnegative reweighing function $w(X)$ with $\expect[w(X) | T=a] = 1$,
    \[ \expect[\ell(Y^a, h(X))] \leq \expect[w(X_i) \ell(Y_i, h(X_i)) | T=a] + E, \]
    where $\sigma^2 \coloneq \Var[\ell(Y^a, h(X))]$ and
    \[ E = \lambda \underbrace{\expect\left[ \left( w(X) \frac{\proba[T=a | X,U]}{\proba[T=a]} - 1 \right)^2 \right]}_{\Delta_{T=a}} + \frac{\sigma^2}{4\lambda}. \]
\end{theorem}

Let's consider some cases. First, the typical scenario in the meta-learners literature, where $w \equiv 1$: in that case, the $\Delta_{T=a}$ term becomes
\[ \Delta_{T=a} = \expect\left[\left(\frac{\proba[T=a | X, U]}{\proba[T=a]} - 1\right)^2\right], \]
i.e., becomes nil when $\proba[T=a | X, U] = \proba[T=a]$ almost everywhere -- which is the case when the data comes from a randomized control trial. When this is not the case, this term (for $w \equiv 1$) is effectfully a measure of how far the study is from being a randomized control trial.

Interestingly, it implies that small local deviations from an RCT (e.g., because for some small part of the population there was some bias in the treatment assignment) should lead to small deviations between the complete and the observed distributions.

Now, let's consider the case where we choose weights to bridge over the distributions -- in particular, consider the optimal weights under the standard ignorability and postivity assumptions sans $U$ (i.e., when there are no unobserved confounders) $w^\star(X) = \dif P_{Y^a,X}/\dif P_{Y,X|T=a} = \proba[T=a]/\proba[T=a | X]$. In that case, the divergence term becomes
\[ \Delta_{T=a} = \expect\left[\left(\frac{\proba[T=a | X, U]}{\proba[T=a | X]} - 1\right)^2\right]. \]
Under these assumptions, we would have that $\proba[T=a | X,U] = \proba[T=a | X]$ and so $\Delta_{T=a} = 0$.
When there are unobserved confounders, $\Delta_{T=a}$ quantifies by how much more certain we become about the treatment mechanism with their information -- i.e., how much the confounders can improve our understanding of the treatment assignment mechanism.

So far, our bounds crucially involve the true propensity scores $\proba[T=a | X,U]$, which are unknown and unobservable in practice.
We can relax the bound from Theorem~\ref{thm:upper-bound-main-theoric-outcome} in order to solve this through the use of a relaxed triangular inequality:
\begin{align*}
    \Delta_{T=a} &= \expect\left[\left(w(X) \frac{\proba[T=a | X, U]}{\proba[T=a]} - 1\right)^2\right]
    \\ &\leq 2 \expect\left[\left(w(X) \frac{\proba[T=a | X, U]}{\proba[T=a]} - w(X) \frac{\ind[T=a]}{\proba[T=a]}\right)^2\right]
    \\ &\quad + 2 \expect\left[\left(w(X) \frac{\ind[T=a]}{\proba[T=a]} - 1\right)^2\right].
\end{align*}
Which can then be rearranged into
\begin{align*}
    \Delta_{T=a} &\leq \frac{2}{\proba[T=a]^2} \expect\left[w(X)^2 \left(\proba[T=a | X, U] - \ind[T=a]\right)^2\right]
    \\ &\quad + 2 \expect\left[\left(w(X) \frac{\ind[T=a]}{\proba[T=a]} - 1\right)^2\right].
\end{align*}
It can be shown that $\proba[T=a | X, U]$ optimizes the term $\expect\left[w(X)^2 \left(\cdot - \ind[T=a]\right)^2\right]$. So, by replacing it with any $\nu(X)$ the bound remains valid, and we obtain that
\begin{align*}
    \Delta_{T=a}
    \leq& \frac{2}{\proba[T=a]^2} \underbrace{\expect\left[w(X)^2 \left(\nu(X) - \ind[T=a]\right)^2\right]}_{\mathrm{BrierScore}_{w^2}(\nu, T)}
    \\ &+ 2 \underbrace{\expect\left[\left(w(X) \frac{\ind[T=a]}{\proba[T=a]} - 1\right)^2\right]}_{D}.
\end{align*}
This is a looser bound than the previous one. For example, while the previous one gave us that, under an RCT, $\Delta_{T=a} = 0$, this time we'd conclude merely that $\Delta_{T=a} \leq C \cdot \mathrm{BrierScore}_{w^2}(\nu, T) + 2D$, which is generally greater than zero.
Nevertheless, we show in Section~\ref{experiments} that this relaxed bound can still be of great practical use and is sufficiently tight in practice.

\begin{theorem}[Empirical upper bound on outcome regression loss in expectation]\label{thm:upper-bound-main-empiric-outcome}
    Let $\Delta_{T=a}$ be as in Theorem~\ref{thm:upper-bound-main-theoric-outcome}. It holds that, for any $\nu$,
    \begin{align*}
        \Delta_{T=a}
            &\leq 2 \expect\left[ \left(\frac{w(X)}{\proba[T=a]}\right)^2 \left(\nu(X) - \ind[T=a]\right)^2 \right]
            \\ &+ 2 \expect\left[ \left(w(X) \frac{\ind[T=a]}{\proba[T=a]} - 1\right)^2 \right].
    \end{align*}
\end{theorem}

Finally, we produce PAC-style~\cite{pac} finite-sample results for the bound in Theorem~\ref{thm:upper-bound-main-empiric-outcome}, concretely showing how its estimation would go in practice.
For brevity, our results here are based on the Rademacher complexity, but we emphasize that our approach works just as easily with other frameworks (e.g., VC, PAC-Bayes, algorithmic stability).

\begin{corollary}[PAC empirical upper bound on outcome regression loss]\label{thm:upper-bound-main-outcome}
    Suppose that $\ell$ is a loss function bounded in $[0, M]$, and $w(X)$ is a nonnegative reweighing function bounded in $[0, w_{\max}]$ with $\expect[w(X) | T=a] = 1$. Then, for any $\lambda > 0$,
    with probability at least $1 - \delta$ over the draw of the training data $(X_i, T_i, Y_i)_{i=1}^n$, for all $h \in \mathcal{H}$ and $\nu \in \mathcal{H}_\nu$,
    \begin{align*}
        \expect[\ell(Y^a, h(X))]
        &\leq \frac{1}{n_{T=a}} \sum_{T_i=a} w(X_i) \ell(Y_i, h(X_i))
        \\ &\kern-4em + \lambda \widehat{\Delta}_{T=a} + \frac{M^2}{16\lambda} + 2 \mathfrak{R}(w \cdot \ell \circ \mathcal{H}) + 2 \mathfrak{R}(\widehat{\Delta} \circ \mathcal{H_\nu})
        \\ &\kern-4em + \left( M w_{\max} + C(w_{\max}) \sqrt{n_{T=a}/n} \right) \sqrt{\frac{\log 2/\delta}{2n_{T=a}}}
    \end{align*}
    where
    \begin{align*}
         \widehat{\Delta}_{T=a} &\coloneq
            \frac{2\lambda}{n} \sum_{i=1}^{n} \left(\frac{w(X_i)}{\proba[T=a]}\right)^2 \left(\nu(X_i) - \ind[T_i=a]\right)^2
            \\ &+ \frac{2\lambda}{n} \sum_{i=1}^{n} \left(w(X_i) \frac{\ind[T_i=a]}{\proba[T=a]} - 1\right)^2,
    \end{align*}
    $\mathfrak{R}(w \cdot \ell \circ \mathcal{H})$ and $\mathfrak{R}(\widehat{\Delta} \circ \mathcal{H_\nu})$ are the Rademacher complexities of $\mathcal{H}$ and $\mathcal{H_\nu}$ composed with their respective loss functions/means,
    and $C(w_{\max})$ is a constant nonnegative quantity defined in the proof.
\end{corollary}

The right-hand-side of this bound has many terms:
first is the empirical observable loss, reweighed by our $w(X)$. Next is the empirical bound for $\Delta_{T=a}$, $\widehat{\Delta}_{T=a}$, along with $M^2/16$ corresponding to the variance term $\sigma^2/4$, both reweighed by the $\lambda$ as in Lemma~\ref{thm:change-of-measure}. Next are the Rademacher complexities $\mathfrak{R}(\mathcal{H})$ and $\mathfrak{R}(\mathcal{H_\nu})$ correcting for the complexity of the algorithms we are fitting. Finally, we have a term bounding the tail of the distribution of the reweighted empirical loss, notably as a factor of $M w_{\max}$, which comes divided by the number of observable samples.

This bound neatly exhibits the cost of using a reweighing $w$ (e.g., the optimal $w^\star$) to bridge over the distributions: if there is some point that is extremely unlikely to be observed, then that point will have a very large weight, increasing $w_{\max}$ substantially. Thankfully, this is not irremediable: by ``simply'' using more data, the effect of this term on the bound can be reduced until it no longer dominates the rest.

\subsection{Causal Meta-learners}\label{sec:bounds-for-causal-metalearners}

We now shift our attention to using our covariates $X$ to predict the individual treatment effects $Y^1 - Y^0$ -- i.e., what would be the benefit of applying the treatment versus doing nothing to this particular individual.

We are particularly interested in causal meta-learners:
procedures that leverage previously-established (or purpose-built) regression algorithms as oracles in an estimating equation.
We focus here on the most common meta-learners available: the T-learner, the S-learner and the X-learner~\cite{x-learner}.
Nevertheless, our analysis is also directly applicable to many other common causal regression methods such as the Causal BART~\cite{causal-bart-1,causal-bart-2} and TARNet/CFRNet~\cite{nns-1}, since these can be seen as special cases of T- or S-learners.

Throughout this section, we focus on losses that satisfy a sort of relaxed triangular inequality:
\begin{assumption}\label{loss-function-assumption}
    The loss function $\ell(\cdot, \cdot)$ can be restated as
    \[ \ell(Y, \widehat{Y}) = \psi(Y - \widehat{Y}), \]
    with $\psi$ satisfying a relaxed subadditive condition: there is some $C$ such that, for any $x, y$, $\psi(x \pm y) \leq C (\psi(x) + \psi(y))$.
\end{assumption}
Most loss functions of interest satisfy Assumptions~\ref{loss-function-assumption}:
\begin{itemize}
    \item Mean Squared Error: take $\psi(x) = x^2$ and $C = 2$.
    \item Mean Absolute Error: take $\psi(x) = \lvert x \rvert$ and $C = 1$.
    \item $\alpha$-Quantile Loss: take $\psi(x) = x\alpha \ind[x \geq 0] - x(1-\alpha) \ind[x < 0]$ and $C = \max\{\alpha,1-\alpha\}/\min\{\alpha,1-\alpha\}$.
    \item 0-1 Loss\footnote{In the case of binary classification, taking both the label and predictions $Y, \widehat{Y}$ to assume numeric values in $\{0, 1\}$.}: take $\psi(x) = \lvert x \rvert$ and $C = 1$.
\end{itemize}

Such an assumption allows us to employ a sort of relaxed triangle inequality in order to decompose the loss of a meta-learner into the losses of its individual components.
This technique has been widely applied in previous works (e.g. \cite{x-learner,prior-work}), though previously restricted to only the mean squared error.
Through the introduction of Assumption~\ref{loss-function-assumption}, we are able to make our results substantially more general -- the implications of which we discuss in Section~\ref{sec:beyond-mse}.

\subsubsection{T-learners and S-learners}\label{sec:tlearner-slearner}

In these types of models, we first obtain two functions $h^1$ and $h^0$ that have been trained to approximate $Y | X, T=1$ and $Y | X, T=0$ respectively.
In the case of T-learners, this is done by independently training $h^1$ and $h^0$ on the samples where $T=1$ and $T=0$, while for S-learners it is done by training a single model $h : \mathcal{X} \times \{0, 1\} \to \mathcal{Y}$ to use $X$ and $T$ to predict $Y$, and taking $h^1 = h(\cdot, 1)$ and $h^0 = h(\cdot, 0)$.
With $h^1$ and $h^0$ in hand, we can predict the individual treatment effect as $h^1(X) - h^0(X)$.

To bound the losses of our predictions $h^1(X) - h^0(X)$, we can separate it in terms of the losses of the individual functions $h^1$ and $h^0$ through the use of Assumption~\ref{loss-function-assumption}:
\begin{align*}
    & \expect[\ell(Y^1 - Y^0, h^1(X) - h^0(X))]
    \\ &\quad \leq C (\expect[\ell(Y^1, h^1(X))] + \expect[\ell(Y^0, h^0(X))])
\end{align*}
From there on, we can use the bounds developed in Section~\ref{sec:bounds-for-outcome-regression} to obtain bounds in terms of the observable losses:

\begin{proposition}[Upper bound on T-/S-learner loss in expectation]\label{thm:upper-bound-main-tlearner}
    Let $\ell$ be a loss function satisfying Assumption~\ref{loss-function-assumption} with constant $C$, and let $w^1(X), w^0(X)$ be nonnegative reweighing functions with $\expect[w^a(X) | T=a] = 1$ for $a = 0, 1$.
    For any $\lambda_1, \lambda_0 > 0$,
    \begin{align*}
        & \expect[\ell(Y^1-Y^0, h^1(X)-h^0(X))]
        \\ &\quad \leq C \Bigl( \expect[w^1(X) )\ell(Y, h^1(X)) | T=1]
        \\ &\qquad + \expect[w^0(X) \ell(Y, h^0(X)) | T=0]
        \\ &\qquad + \underbrace{\lambda_1 \Delta_{T=1} + \lambda_0 \Delta_{T=0}}_{\Delta_{\mathrm{T-learner}}}
                  + \underbrace{\sigma^2_{T=1}/4\lambda_1 + \sigma^2_{T=0}/4\lambda_0}_{\sigma^2/4} \Bigr),
    \end{align*}
    where $\sigma^2_{T=a} \coloneq \Var[\ell(Y^a, h^a(X))]$ and
    \begin{align*}
        \Delta_{T=a}
        &= \expect\left[ \left( w^a(X) \frac{\proba[T=a | X,U]}{\proba[T=a]} - 1 \right)^2 \right]
    \end{align*}
    \begin{align*}
        &\leq \frac{2}{\proba[T=a]^2} \expect\left[ w^a(X)^2 \left( \nu(X) - \ind[T=a] \right)^2 \right]
        \\ &\quad + 2 \expect\left[ \left( w^a(X) \frac{\ind[T=a]}{\proba[T=a]} - 1 \right)^2 \right].
    \end{align*}
\end{proposition}

We can also produce a finite-sample bound akin to Corollary~\ref{thm:upper-bound-main-outcome}. Again, we only present here a bound based on the Rademacher complexity, but the same ideas can be applied to other paradigms.

\begin{corollary}[PAC empirical upper bound on the loss of a T-/S-learner]\label{thm:corollary-tlearner}
    Let $\ell$ be a loss function bounded in $[0, M]$ satisfying Assumption~\ref{loss-function-assumption} with constant $C$ and let $w^1(X)$, $w^0(X)$ be nonnegative reweighing functions bounded in $[0, w_{max}]$ and with $\expect[w^a(X) | T=a] = 1$ for $a = 0, 1$.
    Then, for any $\lambda_1, \lambda_0 > 0$,
    with probability at least $1 - \delta$ over the draw of the training data $(X_i, T_i, Y_i)_{i=1}^n$,
    for all $h^1, h^0 \in \mathcal{H}$ and $\nu \in \mathcal{H_\nu}$,
    \begin{align*}
        & \expect[\ell(Y^1 - Y^0, h^1(X) - h^0(X))]
        \\ &\leq C \Biggl( \frac{1}{n_{T=1}} \sum_{T_i=1} w^1(X_i) \ell(Y_i, h^1(X_i))
        \\ &\qquad + \frac{1}{n_{T=0}} \sum_{T_i=0} w^0(X_i) \ell(Y_i, h^0(X_i))
        \\ &\qquad + \lambda_1 \widehat{\Delta}_{T=1} + \lambda_0 \widehat{\Delta}_{T=0} + \frac{M^2}{16\lambda_1} + \frac{M^2}{16\lambda_0}
        \\ &\qquad + 2 \mathfrak{R}(w^1 \cdot \ell \circ \mathcal{H}) + 2 \mathfrak{R}(w^0 \cdot \ell \circ \mathcal{H}) + 2 \mathfrak{R}(\widehat{\Delta} \circ \mathcal{H_\nu})
        \\ &\qquad + \left( c M w_{\max} + C(w_{\max}) \sqrt{n_{T=\min}/n} \right) \sqrt{\frac{\log 3/\delta}{2n_{T=\min}}} \Biggr)
    \end{align*}
    where
    \begin{align*}
         \widehat{\Delta}_{T=a} &\coloneq
            \frac{2\lambda}{n} \sum_{i=1}^{n} \left(\frac{w^a(X_i)}{\proba[T=a]}\right)^2 \left(\nu(X_i) - \ind[T_i=a]\right)^2
            \\ &+ \frac{2\lambda}{n} \sum_{i=1}^{n} \left(w^a(X_i) \frac{\ind[T_i=a]}{\proba[T=a]} - 1\right)^2,
    \end{align*}
    $n_{T=\min} = \min \{n_{T=1}, n_{T=0}\}$, $\mathfrak{R}(w^a \cdot \ell \circ \mathcal{H})$ and $\mathfrak{R}(\widehat{\Delta} \circ \mathcal{H}_\nu)$ are the Rademacher complexities of $\mathcal{H}$ and $\mathcal{H_\nu}$ composed with their respective loss functions/means, and $c$ and $C(w_{\max})$ are a constant nonnegative quantities defined in the proof, with $1 \leq c \leq 2$.
\end{corollary}

It should be noted that these bounds do not account for interactions between $h^1$ and $h^0$. Quantifying such interactions requires some consideration for the structure of the models themselves, which would lead to losing the mostly-model-agnostic flavor of our results.

\subsubsection{X-learners}\label{sec:xlearner}

An X-learner follows a slightly more involved procedure than the T- and S-learners.
First, one estimates $h^1$ and $h^0$ as in a T-learner. With that, new models $\tau^1$ and $\tau^0$ are trained to regress on $Y^1 - h^0(X)$ and $h^1(X) - Y^0$, respectively, as some sort of pseudo-treatment-effect labels. The individual treatment effect is then estimated as $e(X) \tau^1(X) + (1-e(X)) \tau^0(X)$, for some $e(x)$ (often an estimate of the propensity score $\proba[T=1 | X]$).

Similar to what we did in Section~\ref{sec:tlearner-slearner}, we will bound the loss of the X-learner estimator by separating it in terms of the losses of its individual components $h^1$, $h^0$, $\tau^1$ and $\tau^0$.

For ease of notation, let $\bar{e}(x) = 1 - e(x)$ and $\ell_e(Y, \widehat{Y}) \coloneq \ell(e(X) Y, e(X) \widehat{Y})$. Then, by Assumption~\ref{loss-function-assumption}:
\begin{align*}
    & \expect[\ell(Y^1 - Y^0, e(X) \tau^1(X) + \bar{e}(X) \tau^0(X))]
    \\ &\leq C^2 (\expect[\ell_{\bar{e}}(Y^1, h^1(X))] + \expect[\ell_e(Y^0, h^0(X))]
    \\ &\qquad\quad + \expect[\ell_e(Y^1 - h^0(X), \tau^1(X))]
    \\ &\qquad\quad + \expect[\ell_{\bar{e}}(h^1(X) - Y^0, \tau^0(X))]).
\end{align*}
And again, all that remains is to use the bounds from Section~\ref{sec:bounds-for-outcome-regression} to obtain observable bounds for the complete loss of the X-learner.

\begin{proposition}[Upper bound on X-learner loss in expectation]\label{thm:upper-bound-main-xlearner}
    Let $\ell$ be a loss function satisfying Assumption~\ref{loss-function-assumption} with constant $C$ and let $w^1$, $w^0$ be nonnegative reweighting functions with $\expect[w^a(X) | T=a] = 1$ for $a = 0, 1$.
    For any $\lambda_1, \lambda_0, \lambda_{0,1}, \lambda_{1,0} > 0$,
    \begin{align*}
        & \expect[\ell(Y^1-Y^0, e(X) \tau^1(X) + \bar{e}(X) \tau^0(X))]
        \\ &\quad \leq C^2 \Bigl( \expect[w^1(X) \ell_{\bar{e}}(Y, h^1(X)) | T=1]
        \\ &\quad + \expect[w^0(X) \ell_{e}(Y, h^0(X)) | T=0]
        \\ &\quad + \expect[w^1(X) \ell_e(Y^1 - h^0(X), \tau^1(X)) | T=1]
        \\ &\quad + \expect[w^0(X) \ell_{\bar{e}}(h^1(X) - Y^0, \tau^0(X)) | T=0]
        \\ &\quad + \underbrace{(\lambda_1 + \lambda_{1,0}) \Delta_{T=1} + (\lambda_0 + \lambda_{0,1}) \Delta_{T=0}}_{\Delta_{\mathrm{X-learner}}}
        \\ &\quad + \underbrace{\sigma^2_{T=1}/4\lambda_1 + \sigma^2_{T=0}/4\lambda_0 + \sigma^2_{1,0}/4\lambda_{1,0} + \sigma^2_{0,1}/4\lambda_{0,1}}_{\sigma^2/4} \Bigr),
    \end{align*}
    where $\sigma^2_{T=a} \coloneq \Var[\ell(Y^a, h^a(X))]$, $\sigma^2_{a,b} \coloneq \Var[\ell(Y^a - h^b(X), \tau^a(X))]$ and
    \begin{align*}
        \Delta_{T=a}
        &= \expect\left[ \left( w^a(X) \frac{\proba[T=a | X,U]}{\proba[T=a]} - 1 \right)^2 \right]
        \\ &\leq \frac{2}{\proba[T=a]^2} \expect\left[ w^a(X)^2 \left( \nu(X) - \ind[T=a] \right)^2 \right]
        \\ &\quad + 2 \expect\left[ \left( w^a(X) \frac{\ind[T=a]}{\proba[T=a]} - 1 \right)^2 \right].
    \end{align*}
\end{proposition}

We can also obtain finite-sample bounds akin to our Corollary~\ref{thm:corollary-tlearner}, which can be found in Appendix~\ref{appendix-a} as Corollary~\ref{thm:upper-bound-empirical-xlearner}.

Just as in our analysis of T- and S- learners, we are blind to potential improvements to the bounds due to interactions between the individual regressions, in the spirit of keeping our analysis reasonably model-agnostic.

With both bounds for T/S-learners and for X-learners in hand, a natural question arises: can our theory suggest when one should be preferred over the other?
Since our bounds for these classes of models share many terms, it is not clear a priori whether one method should be preferred over the other. In the end, it appears that any possible advantage originates from the potentially improved inductive biases of the individual regressions. This is consistent with previous findings~\cite{x-learner,acic16}.

\subsection{Beyond the Mean Squared Loss}\label{sec:beyond-mse}

One remarkable aspect of the bounds developed in Section~\ref{sec:bounds-for-causal-metalearners} is how they are general in relation to the loss.
That is, they can limit other losses such as the mean absolute error (for robust regression) and the quantile loss (for quantile regression) just as well as the more typical mean squared loss.
Moreover, minimizing the same loss on each component of the meta-learner implies minimizing the total loss of the meta-learner as a whole.

This is a rather surprising result. For instance, it was previously understood that the conditional quantile treatment effect $Q_\alpha(Y^1 - Y^0 | X)$ was unidentifiable in general. Yet we show that by training a T-learner optimizing the quantile loss, for example, we can estimate it by first approximating
$h^1(X) \approx Q_\alpha(Y^1 | X)$ and $h^0 \approx Q_\alpha(Y^0 | X)$, leading to
\[ h^1(X) - h^0(X) \approx Q_\alpha(Y^1 - Y^0 | X). \]
This is counterintuitive, since $Q_\alpha(Y^1 | X) - Q_\alpha(Y^0 | X) \neq Q_\alpha(Y^1 - Y^0 | X)$ in general! The key to our argument is that we are \emph{not} asserting that $h^1(X) - h^0(X)$ will always converge to $Q_\alpha(Y^1 - Y^0 | X)$. We have bounded the complete quantile loss of the T-learner, but we have never proven that it is even close to the optimal one (which is what would then imply convergence to the conditional quantile).

\begin{figure*}[ht]
    \centering
    \includegraphics[width=.99\textwidth]{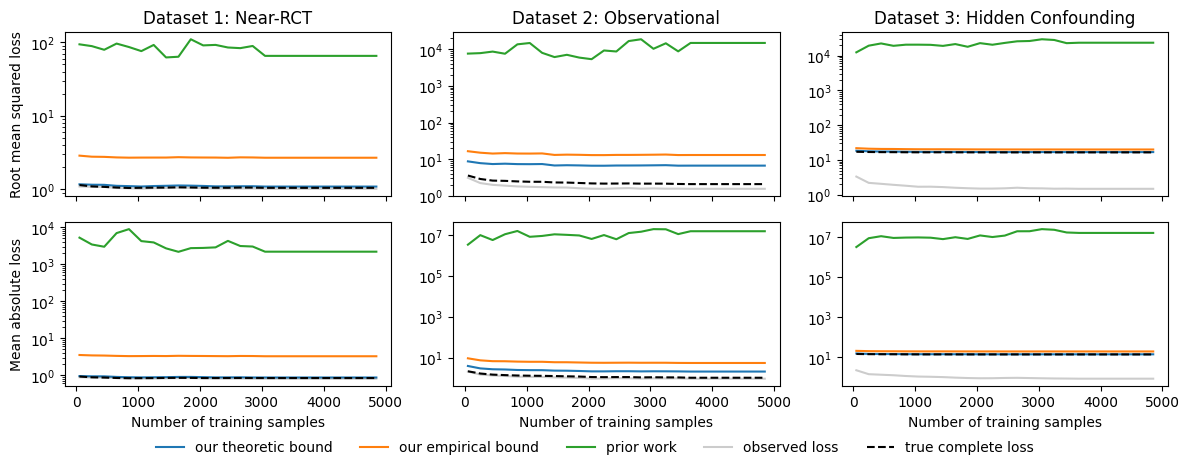}
    \caption{
        \textbf{Tightness of our bounds.}
        Comparison between our bounds and those of \cite{prior-work}, both for the complete loss of the estimation of the potential outcome $Y^1$. Additional images for other tasks (e.g., estimation of treatment effects) are available in Appendix~\ref{sec:more-figures}.
        Our ``theoretic" and ``empirical" bounds correspond in Theorems~\ref{thm:upper-bound-main-theoric-outcome} and \ref{thm:upper-bound-main-empiric-outcome}, and ``prior work'' refers to Corollary~1 of \cite{prior-work}.
        Our theoretic bound is quite tight, being very close to the complete loss (which is unobservable in practice). Our empirical bound, while somewhat looser than the theoretic bound, is still \emph{substantially} tighter than the available prior work.
    }
    \label{fig:figure-1}
\end{figure*}

\begin{figure}[ht]
    \centering
    \includegraphics[width=.45\textwidth]{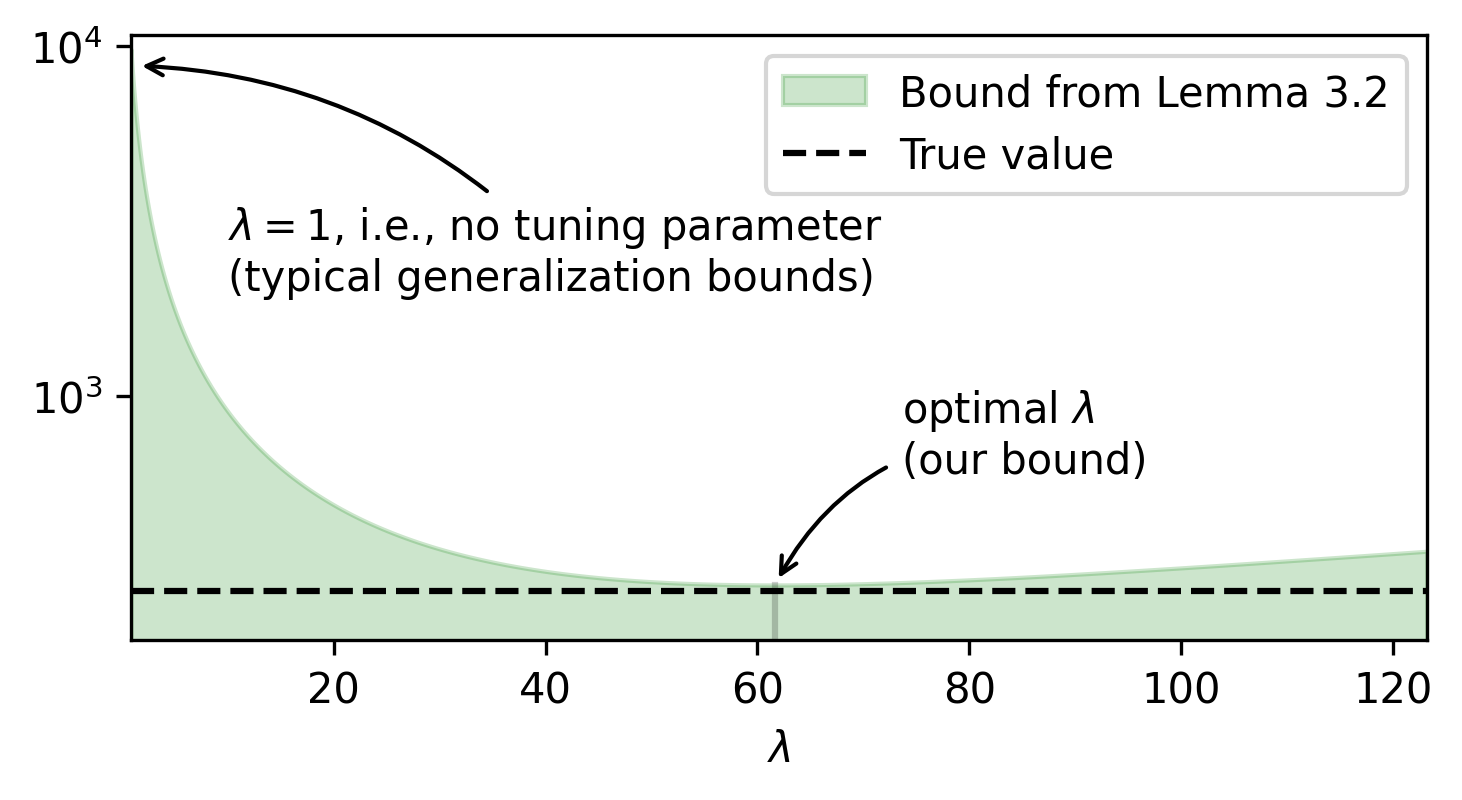}
    \caption{
        \textbf{Importance of the tuning parameter $\lambda$ in Lemma~\ref{thm:change-of-measure}.}
        An illustration of the bound in Lemma~\ref{thm:change-of-measure} (shaded in green) over different values of its tuning parameter $\lambda$.
        Change-of-measure inequalities (e.g., \cite{novel-change-of-measure}) typically do not have a tuning parameter, which corresponds to taking $\lambda = 1$ in our lemma.
        As can be seen in the figure, being able to optimally select $\lambda$ substantially tightens our bounds.
    }
    \label{fig:importance-of-tuning-parameter}
    \vspace{-1.5em}
\end{figure}

Nevertheless, we do show that by improving how well we estimate the individual conditional quantiles (in terms of the quantile loss) we also improve how well we estimate the conditional quantile of the treatment effect. At the limit, if $h^1$ and $h^0$ achieve zero quantile loss (meaning they predict the potential outcomes exactly) and we have a negligible divergence term (e.g., $\Delta = 0$), then the quantile loss for the treatment effect will also be zero (meaning we will predict the individual treatment effects exactly).

\section{Experiments and Applications}
\label{experiments}

In order to assess the tightness and efficacy of our bounds, we empirically evaluate them on three semi-synthetic datasets of increasing difficulty: first simulating a randomized control trial and later in simulations of observational data with fully observed and hidden confounders. Finally, we tackle a real-world application of model selection in a Parkinson's telemonitoring dataset.

Experiments were run on an AMD Ryzen 9 5950X CPU (2.2GHz/5.0GHz, 32 threads) with 64GB of RAM. Nevertheless, the relevant code is lightweight and should easily run on weaker hardware.
More details can be found in Appendix~\ref{sec:experiment-details}, and the code is available
at \url{https://github.com/dccsillag/experiments-causal-generalization-bounds}.

\begin{figure*}[ht]
    \centering
    \includegraphics[width=.99\textwidth]{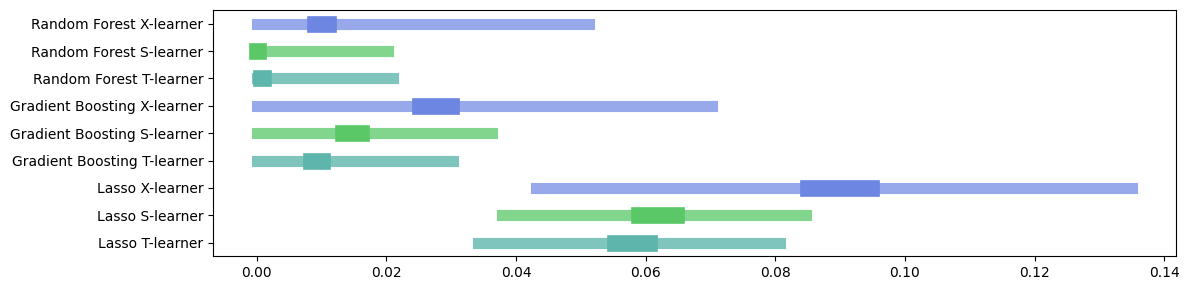}
    \caption{
        \textbf{Application: model selection on real data.}
        The plot compares multiple models for treatment effect estimation: bars correspond to our bounds on the complete loss of the models, while the knobs in the middle correspond to standard bootstrapped confidence intervals for the loss on the observed distributions.
        Note how some models (e.g., R.F. T-learner) appear strictly better than others (e.g., G.B. X-learner) only if ours bounds are not considered. The R.F. and G.B. T/S-learners remain strictly better than the Lasso-based models.
    }
    \label{fig:figure-3}
\end{figure*}

\subsection{Experiments on Semi-Synthetic Data}\label{sec:experiments-semisynthetic}

To evaluate our procedures we use semi-synthetic data (i.e., synthetic data that attempts to reproduce real data) as this allows us to have access to the potential outcomes $Y^1$, $Y^0$.

We use the following datasets, ordered by increasing difficulty.
 \textbf{Learned IHDP}: Results of a randomized control trial simulated with generative models trained on the IHDP~\cite{causal-bart-1} dataset.
 \textbf{ACIC16}: Simulated observational data from \cite{acic16} with fully observed confounding, satisfying ignorability and positivity assumptions;
 \textbf{Confounded ACIC16}: Observational data with significant unobserved confounding and no positivity. This is also the ACIC16 dataset, but modified so that there are hidden confounders. This does not satisfy ignorability nor positivity, making it an extremely challenging dataset to work with.

Figure~\ref{fig:figure-1} demonstrates the tightness of our bounds, especially in comparison to the prior work of~\cite{prior-work}.
In it, we visualize the observed loss, the (unobservable) complete loss, our bounds and the closest corresponding bound from~\cite{prior-work}, varying the number of training samples, all for the task of estimating the potential outcome $Y^1$.

 Our theoretical bound is remarkably tight both in the Near-RCT data and in the observational datasets with hidden confounding, matching the complete loss almost exactly. Furthermore, in the hidden confounding case even our empirical bound -- which is generally looser than the theoretic one -- is exceptionally tight, possibly due to the inapplicability of the positivity assumption.
In the observable dataset with fully observed confounding there is a wider gap between our bounds and the complete loss, though they are all about the same order of magnitude.
On all accounts, the previous bound of \cite{prior-work} is multiple orders of magnitude looser.

A significant contributor to the tightness of our bounds is the tuning parameter $\lambda$ in Lemma~\ref{thm:change-of-measure}.
Figure~\ref{fig:importance-of-tuning-parameter} showcases its importance, displaying the bound from Lemma~\ref{thm:change-of-measure} on our dataset with hidden confounding for various values of $\lambda$. Most existing bounds correspond to taking $\lambda = 1$ (e.g., those of \cite{novel-change-of-measure} and those of \cite{prior-work}, though it is arguable whether such a tuning parameter could even be introduced into their IPM-based bounds).
By allowing optimization over the $\lambda$ parameter we introduced, our bounds can become over 30$\times$ tighter.

\subsection{Application on real data}\label{sec:experiments-real-data}

In this section, we work on top of the Parkinson's Telemonitoring dataset of \cite{parkinsons}, containing features from voice measurements paired with standardized scores describing the progression of the disease, along with the subject's age, gender, and time since recruitment to the study. The goal is to assess the effect of sex on the progression of the disease. There are likely unobserved confounders (e.g., not enough data about the subject to construct a baseline), and at the same time there is enough data to nearly violate positivity (voice data can be a good predictor of sex).

In order to evaluate a causal link between sex and disease progression, we wish to train a model that estimates the individual treatment effect of the standardized UPDRS score on the subject being male vs. female.
To this end, we have trained a myriad of models consisting of T-learners, S-learners and X-learners based on a Lasso, Gradient Boosting and Random Forests.

To assess the quality of our models, we utilize our bounds from Propositions~\ref{thm:upper-bound-main-tlearner} and \ref{thm:upper-bound-main-xlearner}.
We limit the $\Delta$ term via our empirial bounds, and posit a conservative upper bound on the variance of the loss;
we note that, on the observed distributions, the loss typically has a variance of around $10^{-6}$ to $10^{-8}$, leading us to conservatively estimate the variance of the loss as below $10^{-5}$ to give some margin to increased deviation in the complete distribution.

The results of this procedure can be seen in Figure~\ref{fig:figure-3}.
Prior to the ``correction'' given by our bounds, we would have confidently concluded that the Random Forest-based S- and T-learners outperformed all others. After considering the gap between the observed and complete distributions, however, this no longer appears to be the case, with these models now performing roughly on par with the gradient boosting-based meta-learners.
However, even after this correction, it is clear that there is an advantage to using more expressive models over our linear meta-learners, these being confidently outperformed by the Random Forest backed T- and S-learners.
Finally, the use of an X-learner does not seem worth it. They are worse in terms of the observable loss and have wider error bars due to the bound -- one should probably opt here instead for either the Random Forest T- or S-learner.

\section{Conclusion}
\label{conclusion}

In this work, we've introduced generalization bounds for outcome regression and treatment effect estimation, establishing rigorously that modern causal ML procedures \emph{can} properly estimate causal quantities.
Our results hold regardless of hidden confounding or lack of positivity.
Moreover, our bounds show that existing procedures for CATE estimation can be adapted for other tasks, such as estimation of quantiles of individual treatment effects -- which was previously thought to be infeasible -- simply by changing the loss function in the individual optimizations.
We also empirically showcase the tightness and practical utility of our bounds on semi-synthetic and real data, where it is shown to vastly outperform the bounds from the closest matching prior work.
We expect that our results can be of great use not only as motivation for using existing algorithms, but also as a backbone for new algorithms on diverse causal inference tasks.

\section*{Acknowledgements}

We would like to thank the reviewers for their helpful and constructive comments.
This work was partially funded by CNPq and FAPERJ.

\section*{Impact Statement}

This paper presents work whose goal is to advance the field of Machine Learning and Causality.
There are many potential societal consequences of our work through applications to various domains such as medicine and economics, but none which we feel must be specifically highlighted here.

\nocite{novel-change-of-measure,specific-dgps,mohri,scikit-learn,geomloss-wasserstein}

\bibliography{main}
\bibliographystyle{icml2024}

\newpage
\appendix
\onecolumn

\section{Theoretical Results}\label{appendix-a}

\subsection{Change of Measure}

\begin{lemma}[Lemma~\ref{thm:change-of-measure} in the main body]
    Let $P$ and $Q$ be as in Definition~\ref{def:pearson-chi2}.
    For any $\lambda > 0$,
    \begin{align*}
        \expect_{Q}[\phi] - E \leq \expect_{P}[\phi] \leq \expect_{Q}[\phi] + E
    \end{align*}
    where
    \[ E \coloneq \lambda \cdot \chi^2(Q \Vert P) + \frac{1}{4\lambda} \Var_P(\phi). \]
    Moreover, the bound is optimized for
    \[ \lambda^\star = \sqrt{\Var_P(\phi) / 4 \chi^2(Q \Vert P)}. \]
\end{lemma}

\begin{proof}
    We build upon the variational representation framework of \cite{variational-repr} and \cite{novel-change-of-measure}.
    By Lemma 2 of \cite{novel-change-of-measure}, we have that, for all $R$ and $S$, for any $\psi$,
    \[ \expect_Q[\psi] \leq \expect_P[\psi] + \chi^2(Q \Vert P) + \frac{1}{4} \Var_P[\psi]. \]
    Consider two possible choices of $\psi$: $\psi_0 = \lambda \phi$ and $\psi_1 = -\lambda \phi$.
    We then have that:
    \begin{gather*}
        \expect_Q[\lambda \phi] = \lambda \expect_Q[\phi] \leq \expect_P[\lambda \phi] + \chi^2(Q \Vert P) + \frac{1}{4} \Var_P[\lambda \phi] = \lambda \expect_P[\phi] + \chi^2(Q \Vert P) + \frac{1}{4} \lambda^2 \Var_P[\phi]
        \\
        \expect_Q[-\lambda \phi] = -\lambda \expect_Q[\phi] \leq \expect_P[-\lambda \phi] + \chi^2(Q \Vert P) + \frac{1}{4} \Var_P[-\lambda \phi] = -\lambda \expect_P[\phi] + \chi^2(Q \Vert P) + \frac{1}{4} \lambda^2 \Var_P[\phi].
    \end{gather*}
    Rearranging, we obtain that, for any $\lambda > 0$,
    \begin{align*}
        \expect_{Q}[\phi] - \lambda \cdot \chi^2(Q \Vert P) - \frac{1}{4\lambda} \Var_P[\phi] \leq \expect_{P}[\phi] \leq \expect_{Q}[\phi] + \lambda \cdot \chi^2(Q \Vert P) + \frac{1}{4\lambda} \Var_P[\phi].
    \end{align*}
    Finally, this bound is tightest by optimizing
    \[ \lambda^\star = \argmin_{\lambda} \left( \expect_{Q}[\phi] + \lambda \cdot \chi^2(Q \Vert P) + \frac{1}{4\lambda} \Var_P(\phi) \right) = \argmin_{\lambda} \left( \lambda \cdot \chi^2(Q \Vert P) + \frac{1}{4\lambda} \Var_P(\phi) \right). \]
    And taking the derivative in respect to $\lambda$ and equating it to zero gives the desired result.
\end{proof}

\subsection{Bounds in Expectation}\label{sec:suppl-bounds-in-expectation}

\begin{theorem}[Theorem~\ref{thm:upper-bound-main-theoric-outcome} in the main body]
    For any $\lambda > 0$, loss function $\ell(\cdot, \cdot)$ and nonnegative reweighing function $w(X)$ with $\expect[w(X) | T=a] = 1$,
    \[ \expect[\ell(Y^a, h(X))] \leq \expect[w(X_i) \ell(Y_i, h(X_i)) | T=a] + E, \]
    where $\sigma^2 \coloneq \Var[\ell(Y^a, h(X))]$ and
    \[ E = \lambda \underbrace{\expect\left[ \left( w(X) \frac{\proba[T=a | X,U]}{\proba[T=a]} - 1 \right)^2 \right]}_{\Delta_{T=a}} + \frac{\sigma^2}{4\lambda}. \]
\end{theorem}

\begin{proof}
    First note that since $\expect[w(X) | T=a] = 1$, the reweighting $w(X)$ induces a distribution $P_{\widetilde{X}}$ over the $X$ with $\dif P_{\widetilde{X}}/\dif P_{X|T=1} = w(X)$.

    By Lemma~\ref{thm:change-of-measure}:
    \begin{align}\label{goal-prop1}
        &\expect[(Y^a - h(X))^2]
        \leq \expect[w(X) (Y^a - h(X))^2 | T=a] + \lambda \chi^2(P_{Y^a, \widetilde{X} | T=a} \Vert P_{Y^a, X}) + \frac{1}{4\lambda} \Var[(Y^a - h(X))^2];
    \end{align}
    Under SUTVA and ignorability with regards to $X$ and $U$ and using the fact that $\dif P_{\widetilde{X}}/\dif P_{X|T=1} = w(X)$ and a simple application of Bayes' rule, we have that
    \begin{align*}
        \frac{\dif P_{Y,\widetilde{X},U|T=a}}{\dif P_{Y,X,U}}
        &= \frac{\dif P_{Y,U|\widetilde{X},T=a} \cdot \dif P_{\widetilde{X}|T=a}}{\dif P_{Y^a,X,U}}
        = w(X) \frac{\dif P_{Y,U|X,T=a} \cdot \dif P_{X|T=a}}{\dif P_{Y^a,X,U}}
        = w(X) \frac{\dif P_{Y,X,U|T=a}}{\dif P_{Y^a,X,U}}
        \\ &= w(X) \frac{\dif P_{Y^a,X,U|T=a}}{\dif P_{Y^a,X,U}}
        = w(X) \frac{\dif P_{Y^a|X,U,T=a} \cdot \dif P_{X,U|T=a}}{\dif P_{Y^a|X,U} \cdot \dif P_{X,U}}
        = w(X) \frac{\dif P_{Y^a|X,U} \cdot \dif P_{X,U|T=a}}{\dif P_{Y^a|X,U} \cdot \dif P_{X,U}}
        \\ &= w(X) \frac{\dif P_{X,U|T=a}}{\dif P_{X,U}}
        = w(X) \frac{\dif P_{T=a|X,U} \cdot \dif P_{X,U} / \dif P_{T=a}}{\dif P_{X,U}}
        = w(X) \frac{\dif P_{T=a|X,U}}{\dif P_{T=a}}
        \\ &= w(X) \frac{\proba[T=a|X,U]}{\proba[T=a]},
    \end{align*}
    and thus
    \begin{align*}
        \chi^2(P_{Y, \widetilde{X} | T=a} \Vert P_{Y^a, X}) = \expect\left[\left(w(X) \frac{\proba[T=a | X, U]}{\proba[T=a]} - 1\right)^2\right].
    \end{align*}
    Plugging this into Equation~\ref{goal-prop1}, we conclude.
\end{proof}

\begin{theorem}[Theorem~\ref{thm:upper-bound-main-empiric-outcome} in the main body]
    Let $\Delta_{T=a}$ be as in Theorem~\ref{thm:upper-bound-main-theoric-outcome}. It holds that, for any $\nu$,
    \begin{align*}
        \Delta_{T=a}
            &\leq 2 \expect\left[ \left(\frac{w(X)}{\proba[T=a]}\right)^2 \left(\nu(X) - \ind[T=a]\right)^2 \right] + 2 \expect\left[ \left(w(X) \frac{\ind[T=a]}{\proba[T=a]} - 1\right)^2 \right].
    \end{align*}
\end{theorem}

\begin{proof}
    We can relax the bound from Theorem~\ref{thm:upper-bound-main-theoric-outcome} in order to solve this through the use of a relaxed triangular inequality:
    \begin{align*}
        & \Delta_{T=a} = \expect\left[\left(w(X) \frac{\proba[T=a | X, U]}{\proba[T=a]} - 1\right)^2\right]
        \\ &\quad \leq 2 \expect\left[\left(w(X) \frac{\proba[T=a | X, U]}{\proba[T=a]} - w(X) \frac{\ind[T=a]}{\proba[T=a]}\right)^2\right]
        + 2 \expect\left[\left(w(X) \frac{\ind[T=a]}{\proba[T=a]} - 1\right)^2\right]
    \end{align*}
    Which can then be rearranged into
    \begin{align*}
        \frac{2}{\proba[T=a]^2} \expect\left[w(X)^2 \left(\proba[T=a | X, U] - \ind[T=a]\right)^2\right]
        + 2 \expect\left[\left(w(X) \frac{\ind[T=a]}{\proba[T=a]} - 1\right)^2\right].
    \end{align*}
    Note that $\expect[w(X)^2 \left(\nu(X) - \ind[T=a]\right)^2]$ is the Brier Score (or Mean Squared Error) of $\nu(X)$ predicting $\ind[T=a]$, reweighted by the covariates $X$ according to $w(X)$. Therefore, it is minimized by $\nu^\star(X, U) = \expect[\ind[T=a] | X, U] = \proba[T=a | X, U]$.
    Therefore, it holds that
    \[ \forall nu. \quad \expect\left[w(X)^2 \left(\proba[T=a | X, U] - \ind[T=a]\right)^2\right] \leq \expect\left[w(X)^2 \left(\nu(X) - \ind[T=a]\right)^2\right], \]
    And we obtain that
    \begin{align*}
        \Delta_{T=a}
        \leq& \frac{2}{\proba[T=a]^2} \expect\left[w(X)^2 \left(\nu(X) - \ind[T=a]\right)^2\right]
        + 2 \expect\left[\left(w(X) \frac{\ind[T=a]}{\proba[T=a]} - 1\right)^2\right].
    \end{align*}
\end{proof}

Before we continue, let us prove a fundamental lemma about the pinball function:

\begin{lemma}\label{thm:pinball-basic}
    Let $\pinball_\alpha(x) = x \alpha \ind[x \geq 0] - x (1-\alpha) \ind[x < 0]$. Then, for all $x \in \mathbb{R}$ and $\alpha \in (0, 1)$, it holds that
    \begin{align*}
        x \alpha \leq \pinball_\alpha(x) \qquad\qquad \text{and} \qquad\qquad -x (1-\alpha) \leq \pinball_\alpha(x).
    \end{align*}
\end{lemma}

\begin{proof}
    To prove that $x \alpha \leq \pinball_\alpha(x)$, we consider two cases:
    \begin{enumerate}
        \item If $x \geq 0$, then $\pinball_\alpha(x) = x \alpha$ and equality holds trivially.
        \item If $x < 0$, then $x \alpha < 0$. But certainly $\pinball_\alpha(x) = -x (1-\alpha) \geq 0$, so the inequality holds.
    \end{enumerate}
    To prove that $-x (1-\alpha) \leq \pinball_\alpha(x)$ we similarly consider two cases:
    \begin{enumerate}
        \item If $x > 0$, then $-x (1-\alpha) < 0$. But certainly $\pinball_\alpha(x) = x \alpha \geq 0$, so the inequality holds.
        \item If $x \leq 0$, then $\pinball_\alpha(x) = -x (1 - \alpha)$, and the equality holds trivially.
    \end{enumerate}
\end{proof}

\begin{proposition}[Applicability of Assumption~\ref{loss-function-assumption}]
    Let $\ell(\cdot, \cdot)$ be the mean squared error, mean absolute error, quantile loss or 0-1 loss. Then it holds that there exists some $C$ and $\psi$ that satisfy Assumption~\ref{loss-function-assumption} for $\ell$.
\end{proposition}

\begin{proof}
    For the mean squared error, write $\ell(y, \widehat{y}) = (y - \widehat{y})^2$, thus having $\psi(x) = x^2$. Then, by a relaxed triangular inequality, $(y \pm \widehat{y})^2 \leq 2 \cdot ( y^2 + (\pm \widehat{y})^2 ) = 2 \cdot ( y^2 + \widehat{y}^2 )$, meaning we can take $C=2$.

    For the mean absolute error, write $\ell(y, \widehat{y}) = \lvert y - \widehat{y} \rvert$, with $\psi(x) = \lvert x \rvert$. Then, by the standard triangular inequality, $\lvert y \pm \widehat{y} \rvert \leq \lvert y \rvert + \lvert \pm \widehat{y} \rvert = \lvert y \rvert + \lvert \widehat{y} \rvert$, meaning we can take $C=1$.

    For the $\alpha$-quantile loss, write $\ell(y, \widehat{y}) = \pinball_\alpha(y - \widehat{y})$, with $\psi(x) = \pinball_\alpha(x) = x \alpha \ind[x \geq 0] - x (1-\alpha) \ind[x < 0]$.
    To prove the triangular-type inequality, we consider two cases and use Lemma~\ref{thm:pinball-basic}:
    \begin{enumerate}
        \item When $y - \widehat{y} \geq 0$, it holds that
            \begin{align*}
                \pinball_\alpha(y - \widehat{y}) &= (y - \widehat{y}) \alpha = y \alpha + \frac{\alpha}{1-\alpha} [-\widehat{y} (1 - \alpha)] \leq \pinball_\alpha(y) + \frac{\alpha}{1-\alpha} \pinball_\alpha(\widehat{y}).
                \\
                \pinball_\alpha(y + \widehat{y}) &= (y + \widehat{y}) \alpha = y \alpha + \widehat{y} \alpha \leq \pinball_\alpha(y) + \pinball_\alpha(\widehat{y}).
            \end{align*}
        \item When $y - \widehat{y} < 0$, it holds that
            \begin{align*}
                \pinball_\alpha(y - \widehat{y}) &= -(y - \widehat{y}) (1-\alpha) = -y (1-\alpha) + \frac{1-\alpha}{\alpha} [\widehat{y} \alpha] \leq \pinball_\alpha(y) + \frac{1-\alpha}{\alpha} \pinball_\alpha(\widehat{y}).
                \\
                \pinball_\alpha(y + \widehat{y}) &= -(y + \widehat{y}) (1-\alpha) = -y (1-\alpha) - \widehat{y} (1-\alpha) \leq \pinball_\alpha(y) + \pinball_\alpha(\widehat{y}).
            \end{align*}
    \end{enumerate}
    Joining both, we conclude that
    \begin{align*}
        \pinball_\alpha(y \pm \widehat{y})
        &\leq \pinball_\alpha(y) + \max \left\{\frac{\alpha}{1-\alpha}, \frac{1-\alpha}{\alpha}, 1\right\} \pinball_\alpha(\widehat{y})
        \\ &\leq \max \left\{\frac{\alpha}{1-\alpha}, \frac{1-\alpha}{\alpha}, 1\right\} \left( \pinball_\alpha(y) + \pinball_\alpha(\widehat{y}) \right)
        \\ &= \max \left\{\frac{\alpha}{1-\alpha}, \frac{1-\alpha}{\alpha}\right\} \left( \pinball_\alpha(y) + \pinball_\alpha(\widehat{y}) \right)
        \\ &= \frac{\max\{\alpha,1-\alpha\}}{\min\{\alpha,1-\alpha\}} \left( \pinball_\alpha(y) + \pinball_\alpha(\widehat{y}) \right).
    \end{align*}

    Finally, for the 0-1 loss (in the binary case), consider $y, \widehat{y} \in \{0,1\}$ and write $\ell(y, \widehat{y}) = \lvert y - \widehat{y} \rvert$. Then the same logic as in the mean absolute error holds.
\end{proof}

\begin{proposition}[Proposition~\ref{thm:upper-bound-main-tlearner} of the main body]
    Let $\ell$ be a loss function satisfying Assumption~\ref{loss-function-assumption} with constant $C$, and let $w^1(X), w^0(X)$ be nonnegative reweighing functions.
    For any $\lambda_1, \lambda_0 > 0$,
    \begin{align*}
        \expect[\ell(Y^1-Y^0, h^1(X)-h^0(X))]
        &\leq C \Bigl( \expect[w^1(X) )\ell(Y, h^1(X)) | T=1] + \expect[w^0(X) \ell(Y, h^0(X)) | T=0]
        \\ &\qquad \quad + \underbrace{\lambda_1 \Delta_{T=1} + \lambda_0 \Delta_{T=0}}_{\Delta_{\mathrm{T-learner}}}
                  + \underbrace{\sigma^2_{T=1}/4\lambda_1 + \sigma^2_{T=0}/4\lambda_0}_{\sigma^2/4} \Bigr),
    \end{align*}
    Where $\sigma^2_{T=a} \coloneq \Var[\ell(Y^a, h^a(X))]$ and
    \begin{align*}
        \Delta_{T=a}
        &= \expect\left[ \left( w^a(X) \frac{\proba[T=a | X,U]}{\proba[T=a]} - 1 \right)^2 \right].
        \\ &\leq \frac{2}{\proba[T=a]^2} \expect\left[ w^a(X)^2 \left( \nu(X) - \ind[T=a] \right)^2 \right]
        \quad + 2 \expect\left[ \left( w^a(X) \frac{\ind[T=a]}{\proba[T=a]} - 1 \right)^2 \right].
    \end{align*}
\end{proposition}

\begin{proof}
    Using Assumption~\ref{loss-function-assumption},
    \begin{align*}
        \expect[\ell(Y^1-Y^0, h^1(X)-h^0(X))]
        &= \expect[\psi\left((Y^1-Y^0) - (h^1(X)-h^0(X))\right)]
        \\ &= \expect[\psi\left((Y^1 - h^1(X)) - (Y^0 - h^0(X))\right)]
        \\ &\leq \expect[\psi(Y^1 - h^1(X)) + \psi(Y^0 - h^0(X))]
        = \expect[\psi(Y^1 - h^1(X))] + \expect[\psi(Y^0 - h^0(X))].
    \end{align*}
    Then, the rest follows by applying Theorem~\ref{thm:upper-bound-main-theoric-outcome} and \ref{thm:upper-bound-main-empiric-outcome}.
\end{proof}

\begin{proposition}[Proposition~\ref{thm:upper-bound-main-xlearner} of the main body]
    Let $\ell$ be a loss function satisfying Assumption~\ref{loss-function-assumption} with constant $C$ and let $w^1$, $w^0$ be nonnegative reweighting functions.
    For any $\lambda_1, \lambda_0, \lambda_{0,1}, \lambda_{1,0} > 0$,
    \begin{align*}
        \expect[\ell(Y^1-Y^0, e(X) \tau^1(X) + \bar{e}(X) \tau^0(X))]
        &\leq C^2 \Bigl(
            \expect[w^1(X) \ell_{\bar{e}}(Y, h^1(X)) | T=1]
            + \expect[w^0(X) \ell_{e}(Y, h^0(X)) | T=0]
            \\ &\kern-10em + \expect[w^1(X) \ell_e(Y^1 - h^0(X), \tau^1(X)) | T=1]
            + \expect[w^0(X) \ell_{\bar{e}}(h^1(X) - Y^0, \tau^0(X)) | T=0]
            \\ &\kern-10em + \underbrace{(\lambda_1 + \lambda_{1,0}) \Delta_{T=1} + (\lambda_0 + \lambda_{0,1}) \Delta_{T=0}}_{\Delta_{\mathrm{X-learner}}}
            + \underbrace{\sigma^2_{T=1}/4\lambda_1 + \sigma^2_{T=0}/4\lambda_0 + \sigma^2_{1,0}/4\lambda_{1,0} + \sigma^2_{0,1}/4\lambda_{0,1}}_{\sigma^2/4}
        \Bigr),
    \end{align*}
    where $\sigma^2_{T=a} \coloneq \Var[\ell(Y^a, h^a(X))]$, $\sigma^2_{a,b} \coloneq \Var[\ell(Y^a - h^b(X), \tau^a(X))]$ and
    \begin{align*}
        \Delta_{T=a}
        &= \expect\left[ \left( w^a(X) \frac{\proba[T=a | X,U]}{\proba[T=a]} - 1 \right)^2 \right]
        \\ &\leq \frac{2}{\proba[T=a]^2} \expect\left[ w^a(X)^2 \left( \nu(X) - \ind[T=a] \right)^2 \right]
        + 2 \expect\left[ \left( w^a(X) \frac{\ind[T=a]}{\proba[T=a]} - 1 \right)^2 \right].
    \end{align*}
\end{proposition}

\begin{proof}
    Using Assumption~\ref{loss-function-assumption},
    \begin{align*}
        & \expect[\ell(Y^1-Y^0, e(X) \tau^1(X) + (1-e(X)) \tau^0(X))]
        \\ &= \expect[\psi\left((Y^1-Y^0) - (e(X) \tau^1(X) + (1-e(X)) \tau^0(X))\right)]
        \\ &= \expect[\psi\bigl(Y^1 - Y^0 - e(X) (\tau^1(X) - (Y^1 - h^0(X)) + (Y^1 - h^0(X)))
            \\ &\qquad - (1-e(X)) (\tau^0(X) - (h^1(X) - Y^0) + (h^1(X) - Y^0))\bigr)]
        \\ &= \expect[\psi\bigl( Y^1 - Y^0 - e(X) Y^1 + e(X) h^0(X) - (1-e(X)) h^1(X) + (1-e(X)) Y^0
            \\ &\qquad - e(X) (\tau^1(X) - (Y^1 - h^0(X))) - (1-e(X)) (\tau^0(X) - (h^1(X) - Y^0)) \bigr)]
        \\ &= \expect[\psi\bigl( (1-e(X)) (Y^1 - h^1(X)) - e(X) (Y^0 - h^0(X))
            \\ &\qquad - e(X) (\tau^1(X) - (Y^1 - h^0(X))) - (1-e(X)) (\tau^0(X) - (h^1(X) - Y^0)) \bigr)]
        \\ &\leq C ( \expect[\psi\left( (1-e(X)) (Y^1 - h^1(X)) - e(X) (Y^0 - h^0(X)) \right)]
            \\ &\qquad + \expect[\psi\left( e(X) (\tau^1(X) - (Y^1 - h^0(X))) + (1-e(X)) (\tau^0(X) - (h^1(X) - Y^0)) \right)] )
        \\ &\leq C^2 ( \expect[\psi\left( (1-e(X)) (Y^1 - h^1(X)) \right)] + \expect[\psi\left( e(X) (Y^0 - h^0(X)) \right)]
            \\ &\qquad + \expect[\psi\left( e(X) (\tau^1(X) - (Y^1 - h^0(X))) \right)] + \expect[\psi\left( (1-e(X)) (\tau^0(X) - (h^1(X) - Y^0)) \right)] )
        \\ &= C^2 ( \expect[\ell_{\bar{e}}(Y^1, h^1(X))] + \expect[\ell_e(Y^0, h^0(X))] + \expect[\ell_e(Y^1 - h^0(X), \tau^1(X))] + \expect[\ell_{\bar{e}}(h^1(X) - Y^0, \tau^0(X))] )
    \end{align*}
    Then, the rest follows by applying Theorems~\ref{thm:upper-bound-main-theoric-outcome} and \ref{thm:upper-bound-main-empiric-outcome}.
\end{proof}

\subsection{Finite-sample Bounds}\label{sec:more-generalization-bounds}

Throughout this section, we rely on bounds on the random variables within the expectations that appear on the right-hand-side of the results in Section~\ref{sec:suppl-bounds-in-expectation}.

\begin{lemma}\label{thm:outcome-regression-hoeffding-bounds}
    Suppose that $\ell$ is a loss function bounded in $[0, M]$, $w(X)$ is a nonnegative reweighing function bounded in $[0, w_{\max}]$, and that $\nu(X)$ is bounded in $[0, 1]$. It then holds that, for all $x, y$:
    \begin{gather*}
        0 \leq w(x) \ell(y, h(x)) \leq w_{\max} M
        \\
        0 \leq \left( \frac{w(x)}{\proba[T=a]} \right)^2 (\nu(x) - \ind[T=a])^2 + \left( w(x) \frac{\ind[T=a]}{\proba[T=a]} - 1 \right)^2 \leq \left( \frac{w_{\max}}{\proba[T=a]} \right)^2 + \max \left\{ 1, \left( \frac{w_{\max}}{\proba[T=a]} - 1 \right)^2 \right\}.
    \end{gather*}
\end{lemma}

\begin{proof}
    For the first bound, simply note that both $w(x)$ and $\ell(y, h(x))$ are nonnegative, and that
    \[ w(x) \ell(y, h(x)) \leq w_{\max} \ell(y, h(x)) \leq w_{\max} M. \]
    The second bound is slightly more involved.
    \[ 0 \leq \frac{w(x)}{\proba[T=a]} \leq \frac{w_{\max}}{\proba[T=a]} \implies \left( \frac{w(x)}{\proba[T=a]} \right)^2 \leq \left( \frac{w_{\max}}{\proba[T=a]} \right)^2 \]
    and
    \[ -1 \leq \nu(x) - \ind[T=a] \leq 1 \implies 0 \leq (\nu(x) - \ind[T=a])^2 \leq 1. \]
    Together, we get that
    \[ 0 \leq \left( \frac{w(x)}{\proba[T=a]} \right)^2 (\nu(x) - \ind[T=a])^2 \leq \left( \frac{w_{\max}}{\proba[T=a]} \right)^2. \]
    Next,
    \[ -1 \leq w(x) \frac{\ind[T=a]}{\proba[T=a]} - 1 \leq \frac{w_{\max}}{\proba[T=a]} - 1 \implies 0 \leq \left( w(x) \frac{\ind[T=a]}{\proba[T=a]} - 1 \right)^2 \leq \max \left\{ 1, \left( \frac{w_{\max}}{\proba[T=a]} - 1 \right)^2 \right\}. \]
    Combining everything, we get that
    \[ 0 \leq \left( \frac{w(x)}{\proba[T=a]} \right)^2 (\nu(x) - \ind[T=a])^2 + \left( w(x) \frac{\ind[T=a]}{\proba[T=a]} - 1 \right)^2 \leq \left( \frac{w_{\max}}{\proba[T=a]} \right)^2 + \max \left\{ 1, \left( \frac{w_{\max}}{\proba[T=a]} - 1 \right)^2 \right\}. \]
\end{proof}

We base ourselves on a variation of the standard PAC Rademacher Complexity bound presented in~\cite{mohri}:

\begin{lemma}[Theorem 3.3 in \cite{mohri}]\label{thm:rademacher-original}
    Let $\mathcal{G}$ be family of functions mapping from $\mathcal{Z}$ to $[0, 1]$. Then, for any $\delta > 0$, with probability at least $1 - \delta$ over the draw of an i.i.d. sample $(Z_i)_{i=1}^n$, the following holds for all $g \in \mathcal{G}$:
    \[ \expect[g(Z)] \leq \frac{1}{n} \sum_{i=1}^n g(Z_i) + 2 \mathfrak{R}(\mathcal{G}) + \sqrt{\frac{\log 1/\delta}{2n}}. \]
\end{lemma}

\begin{lemma}\label{thm:rademacher-base}
    Let $\mathcal{G}$ be family of functions mapping from $\mathcal{Z}$ to $[0, M]$. Then, for any $\delta > 0$, with probability at least $1 - \delta$ over the draw of an i.i.d. sample $(Z_i)_{i=1}^n$, the following holds for all $g \in \mathcal{G}$:
    \[ \expect[g(Z)] \leq \frac{1}{n} \sum_{i=1}^n g(Z_i) + 2 \mathfrak{R}(\mathcal{G}) + M \sqrt{\frac{\log 1/\delta}{2n}}. \]
\end{lemma}

\begin{proof}
    Writing $g(z) = M^{-1} \widetilde{g}(z)$, we have by the definition of the Rademacher complexity and Lemma~\ref{thm:rademacher-original} that
    \[ \mathfrak{R}(\mathcal{G}) = \expect\left[ \expect_{\mathbf{\sigma}} \left[ \sup_{g \in \mathcal{G}} \frac{1}{n} \sum_{i=1}^n \sigma_i g(Z_i) \right] \right] = \expect\left[ \expect_{\mathbf{\sigma}} \left[ \sup_{\widetilde{g} \in \widetilde{\mathcal{G}}} \frac{1}{n} \sum_{i=1}^n \sigma_i M^{-1} \widetilde{g}(Z_i) \right] \right] = M^{-1} \expect\left[ \expect_{\mathbf{\sigma}} \left[ \sup_{\widetilde{g} \in \widetilde{\mathcal{G}}} \frac{1}{n} \sum_{i=1}^n \sigma_i \widetilde{g}(Z_i) \right] \right] = M^{-1} \mathfrak{R}(\widetilde{\mathcal{G}}) \]
    and so
    \begin{align*}
        & M^{-1} \expect[\widetilde{g}(Z)] \leq M^{-1} \frac{1}{n} \sum_{i=1}^n \widetilde{g}(Z_i) + 2 M^{-1} \mathfrak{R}(\mathcal{G}) + \sqrt{\frac{\log 1/\delta}{2n}}.
        \\ \implies& \expect[\widetilde{g}(Z)] \leq \frac{1}{n} \sum_{i=1}^n \widetilde{g}(Z_i) + 2 \mathfrak{R}(\mathcal{G}) + M \sqrt{\frac{\log 1/\delta}{2n}}.
    \end{align*}
\end{proof}

\begin{corollary}[Corollary~\ref{thm:upper-bound-main-outcome} in the main body]
    Suppose that $\ell$ is a loss function bounded in $[0, M]$ and $w(X)$ is a nonnegative reweighing function bounded in $[0, w_{\max}]$. Then, for any $\lambda > 0$,
    with probability at least $1 - \delta$ over the draw of the training data $(X_i, T_i, Y_i)_{i=1}^n$, for all $h \in \mathcal{H}$ and $\nu \in \mathcal{H_\nu}$,
    \begin{align*}
        \expect[\ell(Y^a, h(X))]
        &\leq \frac{1}{n_{T=a}} \sum_{T_i=a} w(X_i) \ell(Y_i, h(X_i))
        + \lambda \widehat{\Delta}_{T=a} + \frac{M^2}{16\lambda}
        \\ &\quad + 2 \mathfrak{R}(w \cdot \ell \circ \mathcal{H}) + 2 \mathfrak{R}(\widehat{\Delta} \circ \mathcal{H_\nu}) + \left( M w_{\max} + C(w_{\max}) \sqrt{n_{T=a}/n} \right) \sqrt{\frac{\log 2/\delta}{2n_{T=a}}}
    \end{align*}
    where
    \begin{align*}
         \widehat{\Delta}_{T=a} &\coloneq
            \frac{2\lambda}{n} \sum_{i=1}^{n} \left(\frac{w(X_i)}{\proba[T=a]}\right)^2 \left(\nu(X_i) - \ind[T_i=a]\right)^2
            + \frac{2\lambda}{n} \sum_{i=1}^{n} \left(w(X_i) \frac{\ind[T_i=a]}{\proba[T=a]} - 1\right)^2,
    \end{align*}
    $\mathfrak{R}(w \cdot \ell \circ \mathcal{H})$ and $\mathfrak{R}(\widehat{\Delta} \circ \mathcal{H_\nu})$ are the Rademacher complexities of $\mathcal{H}$ and $\mathcal{H_\nu}$ composed with their respective loss functions/means
    and $C(w_{\max})$ is a constant nonnegative quantitity defined in the proof.
\end{corollary}

\begin{proof}
    By Theorems~\ref{thm:upper-bound-main-theoric-outcome} and \ref{thm:upper-bound-main-empiric-outcome},
    \begin{align*}
        &\expect[\ell(Y^a, h(X))] \leq \expect[w(X_i) \ell(Y_i, h(X_i)) | T=a]
        \\ &\quad + 2 \lambda \expect\left[ \left(\frac{w(X)}{\proba[T=a]}\right)^2 \left(\nu(X) - \ind[T=a]\right)^2 \right] + 2 \lambda \expect\left[ \left(w(X) \frac{\ind[T=a]}{\proba[T=a]} - 1\right)^2 \right]
        + \frac{\Var[\ell(Y^a, h(X))]}{4\lambda}.
    \end{align*}
    By Popoviciu's inequality, $\sigma^2 \leq M^2 / 4$ and so
    \begin{align*}
        &\expect[\ell(Y^a, h(X))] \leq \expect[w(X_i) \ell(Y_i, h(X_i)) | T=a]
        \\ &\quad + 2 \lambda \expect\left[ \left(\frac{w(X)}{\proba[T=a]}\right)^2 \left(\nu(X) - \ind[T=a]\right)^2 \right] + 2 \lambda \expect\left[ \left(w(X) \frac{\ind[T=a]}{\proba[T=a]} - 1\right)^2 \right]
        + \frac{M^2}{16\lambda}.
    \end{align*}
    Note that for all $x, y$, it holds that $w(X) \ell(Y, h(X)) \leq w_{\max} M$.
    By Lemma~\ref{thm:rademacher-base} along with Lemma~\ref{thm:outcome-regression-hoeffding-bounds}, with probability of at least $1 - \delta/2$,
    \[ \expect[w(X) \ell(Y, h(X)) | T=a] \leq \frac{1}{n_{T=a}} \sum_{T_i=a} w(X_i) \ell(Y_i, h(X_i)) + 2 \mathfrak{R}(w \cdot \ell \circ \mathcal{H}) + M w_{\max} \sqrt{\frac{\log 2/\delta}{2n_{T=a}}}. \]
    And, also with probability of at least $1 - \delta/2$,
    \begin{align*}
        & 2 \lambda \expect\left[ \left(\frac{w(X)}{\proba[T=a]}\right)^2 \left(\nu(X) - \ind[T=a]\right)^2 \right] + 2 \lambda \expect\left[ \left(w(X) \frac{\ind[T=a]}{\proba[T=a]} - 1\right)^2 \right]
        \\ &\quad = 2 \lambda \expect\left[ \left(\frac{w(X)}{\proba[T=a]}\right)^2 \left(\nu(X) - \ind[T=a]\right)^2 + \left(w(X) \frac{\ind[T=a]}{\proba[T=a]} - 1\right)^2 \right]
        \\ &\quad \leq \frac{2 \lambda}{n} \sum_{i=1}^n \left(\frac{w(X)}{\proba[T=a]}\right)^2 \left(\nu(X) - \ind[T=a]\right)^2 + \frac{2 \lambda}{n} \sum_{i=1}^n \left(w(X) \frac{\ind[T=a]}{\proba[T=a]} - 1\right)^2
        \\ &\qquad \qquad + 2 \mathfrak{R}(\widehat{\Delta} \circ \mathcal{H}_\nu) + \left( \left( \frac{w_{\max}}{\proba[T=a]} \right)^2 + \max \left\{ 1, \left( \frac{w_{\max}}{\proba[T=a]} - 1 \right)^2 \right\} \right) \sqrt{\frac{\log 2/\delta}{2n}}.
    \end{align*}
    Therefore, by an union bound, with probability of at least $1 - \delta$,
    \begin{align*}
        &\expect[\ell(Y^a, h(X))] \leq \frac{1}{n_{T=a}} \sum_{T_i=a} w(X_i) \ell(Y_i, h(X_i))
        + \frac{2 \lambda}{n} \sum_{i=1}^n \left(\frac{w(X)}{\proba[T=a]}\right)^2 \left(\nu(X) - \ind[T=a]\right)^2
        \\ &\quad
        + \frac{2 \lambda}{n} \sum_{i=1}^n \left(w(X) \frac{\ind[T=a]}{\proba[T=a]} - 1\right)^2
        + \frac{M^2}{16\lambda}
        + 2 \mathfrak{R}(w \cdot \ell \circ \mathcal{H})
        + 2 \mathfrak{R}(\widehat{\Delta} \circ \mathcal{H_\nu})
        \\ &\quad
        + M w_{\max} \sqrt{\frac{\log 2/\delta}{2n_{T=a}}}
        + \underbrace{\left( \left( \frac{w_{\max}}{\proba[T=a]} \right)^2 + \max \left\{ 1, \left( \frac{w_{\max}}{\proba[T=a]} - 1 \right)^2 \right\} \right)}_{C(w_{\max})} \sqrt{\frac{\log 2/\delta}{2n}}.
    \end{align*}
    And we conclude by rearranging and observing that
    \begin{align*}
        & M w_{\max} \sqrt{\frac{\log 2/\delta}{2n_{T=a}}} + C(w_{\max}) \sqrt{\frac{\log 2/\delta}{2n}} = M w_{\max} \sqrt{\frac{\log 2/\delta}{2n_{T=a}}} + C(w_{\max}) \sqrt{\frac{n_{T=a}}{n}} \sqrt{\frac{\log 2/\delta}{2n_{T=a}}}
        \\ &\quad = \left(w_{\max} M + C(w_{\max}) \sqrt{\frac{n_{T=a}}{n}}\right) \sqrt{\frac{\log 2/\delta}{2n_{T=a}}}.
    \end{align*}
\end{proof}

\begin{corollary}[Corollary~\ref{thm:upper-bound-main-tlearner} in the main body]
    Let $\ell$ be a loss function bounded in $[0, M]$ satisfying Assumption~\ref{loss-function-assumption} and let $w^1(X), w^0(X)$ be nonnegative reweighing functions bounded in $[0, w_{\max}]$. Then, for any $\lambda_1, \lambda_2 > 0$,
    with probability at least $1 - \delta$ over the draw of the training data $(X_i, T_i, Y_i)_{i=1}^n$, for all $h^1, h^0 \in \mathcal{H}$ and $\nu \in \mathcal{H_\nu}$,
    \begin{align*}
        \expect[\ell(Y^1 - Y^0, h^1(X) - h^0(X))]
        &\leq \frac{1}{n_{T=1}} \sum_{T_i=1} w(X_i) \ell(Y_i, h^1(X_i))
        + \frac{1}{n_{T=0}} \sum_{T_i=0} w(X_i) \ell(Y_i, h^0(X_i))
        \\ &\quad + \lambda_1 \widehat{\Delta}_{T=1} + \lambda_0 \widehat{\Delta}_{T=0} + \frac{M^2}{16\lambda_1} + \frac{M^2}{16\lambda_0} + 2 \mathfrak{R}(w^1 \cdot \ell \circ \mathcal{H}) + 2 \mathfrak{R}(w^0 \cdot \ell \circ \mathcal{H}) + 2 \mathfrak{R}(\widehat{\Delta} \circ \mathcal{H_\nu})
        \\ &\quad + \left( c M w_{\max} + C(w_{\max}) \sqrt{n_{T=\min}/n} \right) \sqrt{\frac{\log 3/\delta}{2n_{T=\min}}}
    \end{align*}
    where
    \begin{align*}
         \widehat{\Delta}_{T=a} &\coloneq
            \frac{2\lambda}{n} \sum_{i=1}^{n} \left(\frac{w(X_i)}{\proba[T=a]}\right)^2 \left(\nu(X_i) - \ind[T_i=a]\right)^2
            + \frac{2\lambda}{n} \sum_{i=1}^{n} \left(w(X_i) \frac{\ind[T_i=a]}{\proba[T=a]} - 1\right)^2,
    \end{align*}
    $n_{T=\min} = \min \{n_{T=1}, n_{T=0}\}$, $\mathfrak{R}(w^a \cdot \ell \circ \mathcal{H})$ and $\mathfrak{R}(\widehat{\Delta} \circ \mathcal{H_\nu})$ are the Rademacher complexities of $\mathcal{H}$ and $\mathcal{H_\nu}$ composed with their respective loss functions/means,
    and $c$ and $C(w_{\max})$ are constant nonnegative quantities defined in the proof, with $1 \leq c \leq 2$.
\end{corollary}

\begin{proof}
    By Propositions~\ref{thm:upper-bound-main-tlearner},
    \begin{align*}
        \expect[\ell(Y^1-Y^0, h^1(X)-h^0(X))]
        &\quad \leq C \Bigl(
            \expect[w^1(X) )\ell(Y, h^1(X)) | T=1]
            + \expect[w^0(X) \ell(Y, h^0(X)) | T=0]
            \\ &\qquad
            + \frac{2 \lambda_1}{\proba[T=1]^2} \expect\left[ w^1(X)^2 \left( \nu(X) - \ind[T=1] \right)^2 \right] + 2 \lambda_1 \expect\left[ \left( w^1(X) \frac{\ind[T=1]}{\proba[T=1]} - 1 \right)^2 \right]
            \\ &\qquad
            + \frac{2 \lambda_0}{\proba[T=0]^2} \expect\left[ w^0(X)^2 \left( \nu(X) - \ind[T=0] \right)^2 \right] + 2 \lambda_0 \expect\left[ \left( w^0(X) \frac{\ind[T=0]}{\proba[T=0]} - 1 \right)^2 \right]
            \\ &\qquad
            + \frac{1}{4\lambda_1} \Var[\ell(Y^1, h^1(X))]
            + \frac{1}{4\lambda_0} \Var[\ell(Y^0, h^0(X))]
        \Bigr).
    \end{align*}
    By Popoviciu's inequality, $\sigma^2 \leq M^2 / 4$ and so
    \begin{align*}
        \expect[\ell(Y^1-Y^0, h^1(X)-h^0(X))]
        &\quad \leq C \Bigl(
            \expect[w^1(X) )\ell(Y, h^1(X)) | T=1]
            + \expect[w^0(X) \ell(Y, h^0(X)) | T=0]
            \\ &\qquad
            + \frac{2 \lambda_1}{\proba[T=1]^2} \expect\left[ w^1(X)^2 \left( \nu(X) - \ind[T=1] \right)^2 \right] + 2 \lambda_1 \expect\left[ \left( w^1(X) \frac{\ind[T=1]}{\proba[T=1]} - 1 \right)^2 \right]
            \\ &\qquad
            + \frac{2 \lambda_0}{\proba[T=0]^2} \expect\left[ w^0(X)^2 \left( \nu(X) - \ind[T=0] \right)^2 \right] + 2 \lambda_0 \expect\left[ \left( w^0(X) \frac{\ind[T=0]}{\proba[T=0]} - 1 \right)^2 \right]
            \\ &\qquad
            + \frac{M^2}{16\lambda_1}
            + \frac{M^2}{16\lambda_0}
        \Bigr).
    \end{align*}
    By Lemma~\ref{thm:rademacher-base} along with Lemma~\ref{thm:outcome-regression-hoeffding-bounds}, with probability of at least $1 - \delta/3$,
    \[ \expect[w^a(X) \ell(Y, h^a(X)) | T=a] \leq \frac{1}{n_{T=a}} \sum_{T_i=a} w^a(X_i) \ell(Y_i, h^a(X_i)) + 2 \mathfrak{R}(w^a \cdot \ell \circ \mathcal{H}) + M w_{\max} \sqrt{\frac{\log 3/\delta}{2n_{T=a}}}. \]
    And, also with probability of at least $1 - \delta/3$,
    \begin{align*}
        & \sum_{a \in \{0,1\}} \left( 2 \lambda \expect\left[ \left(\frac{w(X)}{\proba[T=a]}\right)^2 \left(\nu(X) - \ind[T=a]\right)^2 \right] + 2 \lambda \expect\left[ \left(w(X) \frac{\ind[T=a]}{\proba[T=a]} - 1\right)^2 \right] \right)
        \\ &\quad = 2 \lambda \expect\left[ \sum_{a \in \{0,1\}} \left( \left(\frac{w(X)}{\proba[T=a]}\right)^2 \left(\nu(X) - \ind[T=a]\right)^2 + \left(w(X) \frac{\ind[T=a]}{\proba[T=a]} - 1\right)^2 \right) \right]
        \\ &\quad \leq \sum_{a \in \{0,1\}} \left( \frac{2 \lambda}{n} \sum_{i=1}^n \left(\frac{w(X)}{\proba[T=a]}\right)^2 \left(\nu(X) - \ind[T=a]\right)^2 + \frac{2 \lambda}{n} \sum_{i=1}^n \left(w(X) \frac{\ind[T=a]}{\proba[T=a]} - 1\right)^2 \right)
        \\ &\qquad \qquad + 2 \mathfrak{R}(\widehat{\Delta} \circ \mathcal{H}_\nu) + \left(\sum_{a \in \{0,1\}} \left( \frac{w_{\max}}{\proba[T=a]} \right)^2 + \sum_{a \in \{0,1\}} \max \left\{ 1, \left( \frac{w_{\max}}{\proba[T=a]} - 1 \right)^2 \right\} \right) \sqrt{\frac{\log 3/\delta}{2n}}.
    \end{align*}
    Therefore, by an union bound, with probability of at least $1 - \delta$,
    \begin{align*}
        &\expect[\ell(Y^a, h^1(X) - h^0(X))] \leq \frac{1}{n_{T=1}} \sum_{T_i=1} w^1(X_i) \ell(Y_i, h^1(X_i))
        + \frac{1}{n_{T=0}} \sum_{T_i=0} w^0(X_i) \ell(Y_i, h^0(X_i))
        \\ &\quad
        + \frac{2 \lambda}{n} \sum_{i=1}^n \left(\frac{w^1(X)}{\proba[T=1]}\right)^2 \left(\nu(X) - \ind[T=1]\right)^2
        + \frac{2 \lambda}{n} \sum_{i=1}^n \left(w^1(X) \frac{\ind[T=1]}{\proba[T=1]} - 1\right)^2
        \\ &\quad
        + \frac{2 \lambda}{n} \sum_{i=1}^n \left(\frac{w^0(X)}{\proba[T=0]}\right)^2 \left(\nu(X) - \ind[T=0]\right)^2
        + \frac{2 \lambda}{n} \sum_{i=1}^n \left(w^0(X) \frac{\ind[T=0]}{\proba[T=0]} - 1\right)^2
        \\ &\quad
        + \frac{M^2}{16\lambda_1} + \frac{M^2}{16\lambda_2}
        + 2 \mathfrak{R}(w^1 \cdot \ell \circ \mathcal{H}) + 2 \mathfrak{R}(w^0 \cdot \ell \circ \mathcal{H}) + 2 \mathfrak{R}(\widehat{\Delta} \circ \mathcal{H_\nu})
        \\ &\quad
        + M w_{\max} \sqrt{\frac{\log 3/\delta}{2n_{T=1}}}
        + M w_{\max} \sqrt{\frac{\log 3/\delta}{2n_{T=0}}}
        + \underbrace{\left(\sum_{a \in \{0,1\}} \left( \frac{w_{\max}}{\proba[T=a]} \right)^2 + \sum_{a \in \{0,1\}} \max \left\{ 1, \left( \frac{w_{\max}}{\proba[T=a]} - 1 \right)^2 \right\} \right)}_{C(w_{\max})} \sqrt{\frac{\log 3/\delta}{2n}}.
    \end{align*}
    And we conclude by rearranging and observing that, assuming without loss of generality that $n_{T=1} = \min \{n_{T=1}, n_{T=0}\}$,
    \begin{align*}
        & M w_{\max} \sqrt{\frac{\log 3/\delta}{2n_{T=1}}} + M w_{\max} \sqrt{\frac{\log 3/\delta}{2n_{T=0}}} + C(w_{\max}) \sqrt{\frac{\log 3/\delta}{2n}}
        \\ &\quad = M w_{\max} \sqrt{\frac{\log 3/\delta}{2n_{T=\min}}} + M w_{\max} \sqrt{\frac{n_{T=\min}}{n_{T=0}}} \sqrt{\frac{\log 3/\delta}{2n_{T=\min}}} + C(w_{\max}) \sqrt{\frac{n_{T=\min}}{n}} \sqrt{\frac{\log 3/\delta}{2n_{T=\min}}}
        \\ &\quad = \left( M w_{\max} \underbrace{\left(1 + \sqrt{\frac{n_{T=\min}}{n_{T=0}}}\right)}_{c} + C(w_{\max}) \sqrt{\frac{n_{T=\min}}{n}} \right) \sqrt{\frac{\log 3/\delta}{2n_{T=\min}}}
    \end{align*}
    and that $0 \leq \sqrt{n_{T=\min}/n_{T=0}} \leq 1$.
\end{proof}

\begin{corollary}\label{thm:upper-bound-empirical-xlearner}
    Let $\ell$ be a loss function bounded in $[0, M]$ satisfying Assumption~\ref{loss-function-assumption} and let $w^1(X), w^0(X)$ be nonnegative reweighing functions bounded in $[0, w_{\max}]$. Then, for any $\lambda_1, \lambda_0 > 0$,
    with probability at least $1 - \delta$ over the draw of the training data $(X_i, T_i, Y_i)_{i=1}^n$, for all $h^1, h^0, \tau^1, \tau^0 \in \mathcal{H}$ and $\nu \in \mathcal{H_\nu}$,
    \begin{align*}
        \expect[\ell(Y^1-Y^0, e(X) \tau^1(X) + \bar{e}(X) \tau^0(X))]
        &\leq C^2 \Biggl(
            \frac{1}{n_{T=1}} \sum_{T_i=1} w(X_i) \ell_{\bar{e}}(Y_i, h^1(X_i))
            + \frac{1}{n_{T=0}} \sum_{T_i=0} w(X_i) \ell_e(Y_i, h^0(X_i))
            \\ &\kern-10em + \frac{1}{n_{T=1}} \sum_{T_i=1} w(X_i) \ell_e(Y_i - h^0(X_i), \tau^1(X_i))
            + \frac{1}{n_{T=0}} \sum_{T_i=0} w(X_i) \ell_{\bar{e}}(h^1(X_i) - Y_i, \tau^0(X_i))
            \\ &\kern-10em + \lambda_1 \widehat{\Delta}_{T=1} + \lambda_0 \widehat{\Delta}_{T=0} + \frac{M^2}{4\lambda_1} + \frac{M^2}{4\lambda_0}
            + 4 \mathfrak{R}(w^1 \cdot \ell \circ \mathcal{H}) + 4 \mathfrak{R}(w^0 \cdot \ell \circ \mathcal{H}) + 2 \mathfrak{R}(\widehat{\Delta} \circ \mathcal{H_\nu})
            \\ &\kern-10em + \left( 2 c M w_{\max} + C(w_{\max}) \sqrt{n_{T=\min}/n} \right) \sqrt{\frac{\log 5/\delta}{2n_{T=\min}}}
        \Biggr)
    \end{align*}
    where
    \begin{align*}
         \widehat{\Delta}_{T=a} &\coloneq
            \frac{2\lambda}{n} \sum_{i=1}^{n} \left(\frac{w(X_i)}{\proba[T=a]}\right)^2 \left(\nu(X_i) - \ind[T_i=a]\right)^2
            + \frac{2\lambda}{n} \sum_{i=1}^{n} \left(w(X_i) \frac{\ind[T_i=a]}{\proba[T=a]} - 1\right)^2,
    \end{align*}
    $n_{T=\min} = \min \{n_{T=1}, n_{T=0}\}$, $\mathfrak{R}(w^a \cdot \ell \circ \mathcal{H})$ and $\mathfrak{R}(\widehat{\Delta} \circ \mathcal{H_\nu})$ are the Rademacher complexities of $\mathcal{H}$ and $\mathcal{H_\nu}$ composed with their respective loss functions/means,
    and $c$ and $C(w_{\max})$ are constant nonnegative quantities defined in the proof, with $1 \leq c \leq 2$.
\end{corollary}

\begin{proof}
    By Proposition~\ref{thm:upper-bound-main-xlearner},
    \begin{align*}
        \expect[\ell(Y^1-Y^0, e(X) \tau^1(X) + \bar{e}(X) \tau^0(X))]
        &\leq C^2 \Bigl(
            \expect[w^1(X) \ell_{\bar{e}}(Y, h^1(X)) | T=1]
            + \expect[w^0(X) \ell_{e}(Y, h^0(X)) | T=0]
            \\ &\kern-10em + \expect[w^1(X) \ell_e(Y^1 - h^0(X), \tau^1(X)) | T=1]
            + \expect[w^0(X) \ell_{\bar{e}}(h^1(X) - Y^0, \tau^0(X)) | T=0]
            \\ &\kern-10em + \left(\frac{\lambda_1}{2} + \frac{\lambda_1}{2}\right) \Delta_{T=1} + \left(\frac{\lambda_0}{2} + \frac{\lambda_0}{2}\right) \Delta_{T=0}
            + \frac{\sigma^2_{T=1}}{4\lambda_1/2} + \frac{\sigma^2_{T=0}}{4\lambda_0/2} + \frac{\sigma^2_{1,0}}{4\lambda_1/2} + \frac{\sigma^2_{0,1}}{4\lambda_0/2}
        \Bigr),
    \end{align*}
    By Popoviciu's inequality, $\sigma^2 \leq M^2 / 4$ and so
    \begin{align*}
        \expect[\ell(Y^1-Y^0, e(X) \tau^1(X) + \bar{e}(X) \tau^0(X))]
        &\leq C^2 \Bigl(
            \expect[w^1(X) \ell_{\bar{e}}(Y, h^1(X)) | T=1]
            + \expect[w^0(X) \ell_{e}(Y, h^0(X)) | T=0]
            \\ &\kern-10em + \expect[w^1(X) \ell_e(Y^1 - h^0(X), \tau^1(X)) | T=1]
            + \expect[w^0(X) \ell_{\bar{e}}(h^1(X) - Y^0, \tau^0(X)) | T=0]
            \\ &\kern-10em + \lambda_1 \Delta_{T=1} + \lambda_0 \Delta_{T=0}
            + \frac{M^2}{4\lambda_1} + \frac{M^2}{4\lambda_0}
        \Bigr),
    \end{align*}
    By Lemma~\ref{thm:rademacher-base} along with Lemma~\ref{thm:outcome-regression-hoeffding-bounds}, with probability of at least $1 - \delta/5$,
    \begin{align*}
        \expect[w^1(X) \ell_{\bar{e}}(Y, h^1(X)) | T=1] &\leq \frac{1}{n_{T=1}} \sum_{T_i=1} w^1(X_i) \ell_{\bar{e}}(Y_i, h^1(X_i)) + 2 \mathfrak{R}(w^1 \cdot \ell \circ \mathcal{H}) + M w_{\max} \sqrt{\frac{\log 5/\delta}{2n_{T=1}}}.
        \\
        \expect[w^0(X) \ell_{e}(Y, h^0(X)) | T=0] &\leq \frac{1}{n_{T=0}} \sum_{T_i=0} w^0(X_i) \ell_{e}(Y_i, h^0(X_i)) + 2 \mathfrak{R}(w^0 \cdot \ell \circ \mathcal{H}) + M w_{\max} \sqrt{\frac{\log 5/\delta}{2n_{T=0}}}.
    \end{align*}
    Also with probability of at least $1 - \delta/5$,
    \begin{align*}
        \expect[w^1(X) \ell_e(Y^1 - h^0(X), \tau^1(X)) | T=1] &\leq \frac{1}{n_{T=1}} \sum_{T_i=1} w^1(X_i) \ell_e(Y^1 - h^0(X), \tau^1(X)) + 2 \mathfrak{R}(w^1 \cdot \ell \circ \mathcal{H}) + M w_{\max} \sqrt{\frac{\log 5/\delta}{2n_{T=1}}}.
        \\
        \expect[w^0(X) \ell_{\bar{e}}(h^1(X) - Y^0, \tau^0(X)) | T=0] &\leq \frac{1}{n_{T=0}} \sum_{T_i=0} w^0(X_i) \ell_{\bar{e}}(h^1(X) - Y^0, \tau^0(X)) + 2 \mathfrak{R}(w^0 \cdot \ell \circ \mathcal{H}) + M w_{\max} \sqrt{\frac{\log 5/\delta}{2n_{T=0}}}.
    \end{align*}
    And, likewise,
    \begin{align*}
        & \sum_{a \in \{0,1\}} \left( 2 \lambda \expect\left[ \left(\frac{w(X)}{\proba[T=a]}\right)^2 \left(\nu(X) - \ind[T=a]\right)^2 \right] + 2 \lambda \expect\left[ \left(w(X) \frac{\ind[T=a]}{\proba[T=a]} - 1\right)^2 \right] \right)
        \\ &\quad = 2 \lambda \expect\left[ \sum_{a \in \{0,1\}} \left( \left(\frac{w(X)}{\proba[T=a]}\right)^2 \left(\nu(X) - \ind[T=a]\right)^2 + \left(w(X) \frac{\ind[T=a]}{\proba[T=a]} - 1\right)^2 \right) \right]
        \\ &\quad \leq \sum_{a \in \{0,1\}} \left( \frac{2 \lambda}{n} \sum_{i=1}^n \left(\frac{w(X)}{\proba[T=a]}\right)^2 \left(\nu(X) - \ind[T=a]\right)^2 + \frac{2 \lambda}{n} \sum_{i=1}^n \left(w(X) \frac{\ind[T=a]}{\proba[T=a]} - 1\right)^2 \right)
        \\ &\qquad \qquad + 2 \mathfrak{R}(\widehat{\Delta} \circ \mathcal{H}_\nu) + \left(\sum_{a \in \{0,1\}} \left( \frac{w_{\max}}{\proba[T=a]} \right)^2 + \sum_{a \in \{0,1\}} \max \left\{ 1, \left( \frac{w_{\max}}{\proba[T=a]} - 1 \right)^2 \right\} \right) \sqrt{\frac{\log 5/\delta}{2n}}.
    \end{align*}
    Therefore, by an union bound, with probability of at least $1 - \delta$,
    \begin{align*}
        &\expect[\ell(Y^a, e(X) \tau^1(X) + \bar{e}(X) \tau^0(X))] \leq
        \frac{1}{n_{T=1}} \sum_{T_i=1} w^1(X_i) \ell_{\bar{e}}(Y_i, h^1(X_i))
        + \frac{1}{n_{T=0}} \sum_{T_i=0} w^0(X_i) \ell_{e}(Y_i, h^0(X_i))
        \\ &\qquad
        + \frac{1}{n_{T=1}} \sum_{T_i=1} w^1(X_i) \ell_e(Y^1 - h^0(X), \tau^1(X))
        + \frac{1}{n_{T=0}} \sum_{T_i=0} w^0(X_i) \ell_{\bar{e}}(h^1(X) - Y^0, \tau^0(X))
        \\ &\qquad
        + \sum_{a \in \{0,1\}} \left( \frac{2 \lambda}{n} \sum_{i=1}^n \left(\frac{w(X)}{\proba[T=a]}\right)^2 \left(\nu(X) - \ind[T=a]\right)^2
        + \frac{2 \lambda}{n} \sum_{i=1}^n \left(w(X) \frac{\ind[T=a]}{\proba[T=a]} - 1\right)^2 \right)
        \\ &\qquad
        + 4 \mathfrak{R}(w^1 \cdot \ell \circ \mathcal{H})
        + 4 \mathfrak{R}(w^0 \cdot \ell \circ \mathcal{H})
        + 2 \mathfrak{R}(\widehat{\Delta} \circ \mathcal{H}_\nu)
        \\ &\qquad
        + 2 M w_{\max} \sqrt{\frac{\log 5/\delta}{2n_{T=1}}}
        + 2 M w_{\max} \sqrt{\frac{\log 5/\delta}{2n_{T=0}}}
        \\ &\qquad
        + \underbrace{\left(\sum_{a \in \{0,1\}} \left( \frac{w_{\max}}{\proba[T=a]} \right)^2
        + \sum_{a \in \{0,1\}} \max \left\{ 1, \left( \frac{w_{\max}}{\proba[T=a]} - 1 \right)^2 \right\} \right)}_{C(w_{\max})} \sqrt{\frac{\log 5/\delta}{2n}}.
    \end{align*}
    And we conclude by rearranging and observing that, assuming without loss of generality that $n_{T=1} = \min \{n_{T=1}, n_{T=0}\}$,
    \begin{align*}
        & 2 M w_{\max} \sqrt{\frac{\log 5/\delta}{2n_{T=1}}} + 2 M w_{\max} \sqrt{\frac{\log 5/\delta}{2n_{T=0}}} + C(w_{\max}) \sqrt{\frac{\log 5/\delta}{2n}}
        \\ &\quad = 2 M w_{\max} \sqrt{\frac{\log 5/\delta}{2n_{T=\min}}} + 2 M w_{\max} \sqrt{\frac{n_{T=\min}}{n_{T=0}}} \sqrt{\frac{\log 5/\delta}{2n_{T=\min}}} + C(w_{\max}) \sqrt{\frac{n_{T=\min}}{n}} \sqrt{\frac{\log 5/\delta}{2n_{T=\min}}}
        \\ &\quad = \left( 2 M w_{\max} \underbrace{\left(1 + \sqrt{\frac{n_{T=\min}}{n_{T=0}}}\right)}_{c} + C(w_{\max}) \sqrt{\frac{n_{T=\min}}{n}} \right) \sqrt{\frac{\log 5/\delta}{2n_{T=\min}}}
    \end{align*}
    and that $0 \leq \sqrt{n_{T=\min}/n_{T=0}} \leq 1$.
\end{proof}

\vfill

\pagebreak
\section{Details About the Experiments}\label{sec:experiment-details}

\subsection{Details About the Simulated Data}\label{sec:more-details-data}

\subsubsection{Learned IHDP}

The IHDP dataset~\cite{causal-bart-1} are the results of a randomized control trial.
In order to be able to simulate the potential outcomes, we train generative models on it and use these generative models as our data generating process.

First, we train a Forest Diffusion~\cite{forest-diffusion} model to generate the covariates $X$.
We then train two Gaussian Processes: one to predict $Y^1$ from $X$ based on samples from $Y | X, T=1$, and another to predict $Y^0$ from $X$ based on samples from $Y | X, T=0$. The use of Gaussian Processes here allows us to sample from the predictive distribution, introducing variability into the potential outcomes, an essential element of real data.
Finally, a calibrated random forest model is trained to predict the treatment assignments from $X$.

Such data is guaranteed to satisfy ignorability, since the treatment assignment is determined independenly from the potential outcomes, given the covariates $X$. Moreover, since the original data is an RCT, it is expected that the data will roughly resemble an RCT (but not exactly, since the treatment assignment is allowed to vary over the $X$).

\subsubsection{ACIC16}

This is the data from the 2016 edition of the Atlantic Causal Inference Competition~\cite{acic16},
widely used in previous works.

It is synthetic, providing even the potential outcomes, but crafted to resemble real data. Nevertheless, it is guaranteed to satisfy the standard ignorability (i.e., no hidden confounding) and positivity assumptions.

\subsubsection{Confounded ACIC16}

This is built on top of the data from ACIC16. Having generated $X, T, Y^1, Y^0$ as in ACIC16, we now modify the potential outcomes in order to violate positivity and ignorability by making it so that when $T=0$, $Y^1$ is offset by $-20$.
This makes it so that $T$ (which is not part of $X$) becomes a hidden confounder, and positivity is violated (since certain outcomes happen only in the counterfactuals).
Finally, this modification makes it so that the true treatment propensities (accounting for the unobserved confounder) are exactly equal to $T$.

\subsection{Learning reweightings and the $\nu$}

\subsubsection{Learning to reweight}

We consider two options to learn the reweighting functions $w(X)$. One is to simply use the constant reweighting $w \equiv 1$, which is a surprisingly strong option.

Another option is to approximate the ``optimal'' weights given by
$$ w^{a \star} (X) = \frac{\proba[T=a]}{\proba[T=a | X]}. $$
Were we to use these precise weights under no hidden confounding and positivity, we would fully eliminate the gap between the observed and complete distributions, since for any $\phi(X)$,
$$ \expect[ w^{a \star}(X) \cdot \phi(X) | T=a] = \expect[ \phi(X) ]. $$
To approximate $w^{a \star}$, we first train a classifier $\widehat{e}(X)$ to estimate $\proba[T=a | X]$ and estimate the probability $\proba[T=a]$ via its sample mean as $\widehat{p}_{T=a}$, and produce the following intermediate unnormalized approximation $\widetilde{w}^a(X)$ for the weights:
$$ \widetilde{w}^a(X) = \widehat{p}_{T=a} / \widehat{e}(X). $$
To ensure that $\expect[\widehat{w}^a(X) | T=a] = 1$ (as required by our bounds), we normalize $\widetilde{w}^a(X)$ by its sample mean $M$:
$$ \widehat{w}^a(X) = \widetilde{w}^a(X) / \widehat{M}, \qquad \qquad \widehat{M} = \widehat{\expect}[\widetilde{w}^a(X)]. $$

\subsubsection{Learning a $\nu$}

The $\nu(X)$ is a probabilistic classification model for the treatment assignment (possibly the very same one used for weight estimation), with $\nu(X)$ being the predicted probability of $T=a$ given $X$ (i.e., $\nu(X) \approx \proba[T=a | X]$).
Since $\nu$ appears in the bound within a classification loss of the model for predicting T (the Brier score), a better classification model (i.e., a better $\nu$) means a better bound.

\subsection{Details about the figures}

In Figure~\ref{fig:figure-1},
all means were estimated via the standard empirical mean with an adequately high number of samples.
The underlying models used were Random Forests as per Scikit-Learn's implementation~\cite{scikit-learn} with the default hyperparameters.
For the prior work bound of \cite{prior-work}, the estimation of the Wasserstein distance (which is highly nontrivial) was done with \verb|GeomLoss|, which implements the method from~\cite{geomloss-wasserstein}.

Figure~\ref{fig:importance-of-tuning-parameter} was computed on the Confounded ACIC16 dataset on the loss of outcome regression of $Y^1$, with the relevant means being estimated via the standard empirical mean from a sufficiently large number of samples.

In Figure~\ref{fig:figure-3}, none of the models use sample reweighting. An alternate version of the figure including such models can be found in Section~\ref{sec:more-figures}. The loss used in the figure is the mean squared loss.

\pagebreak
\section{More figures}\label{sec:more-figures}

In what follows:

\begin{itemize}
    \item Quantile loss refers to the quantile loss with $\alpha = 0.8$;
    \item 0-1 loss refers to the 0-1 loss for predicting whether the target is above its median value.
\end{itemize}

\begin{figure}[ht]
    \centering
    \includegraphics[width=.99\textwidth]{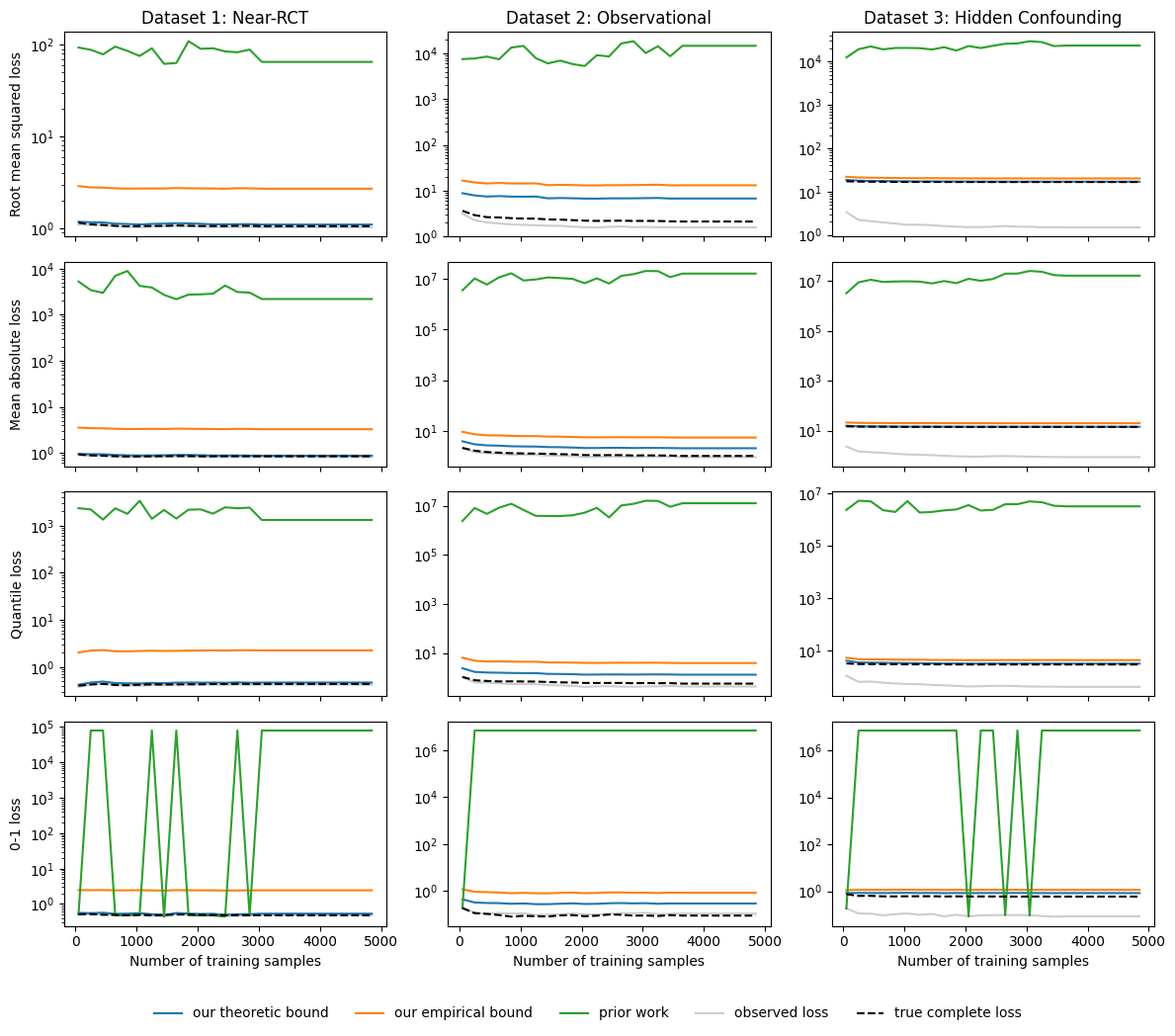}
    \caption{
        \textbf{Alternate version of Figure~\ref{fig:figure-1} for outcome regression of $Y^1$ and including more losses.}
    }
\end{figure}

\begin{figure}[ht]
    \centering
    \includegraphics[width=.99\textwidth]{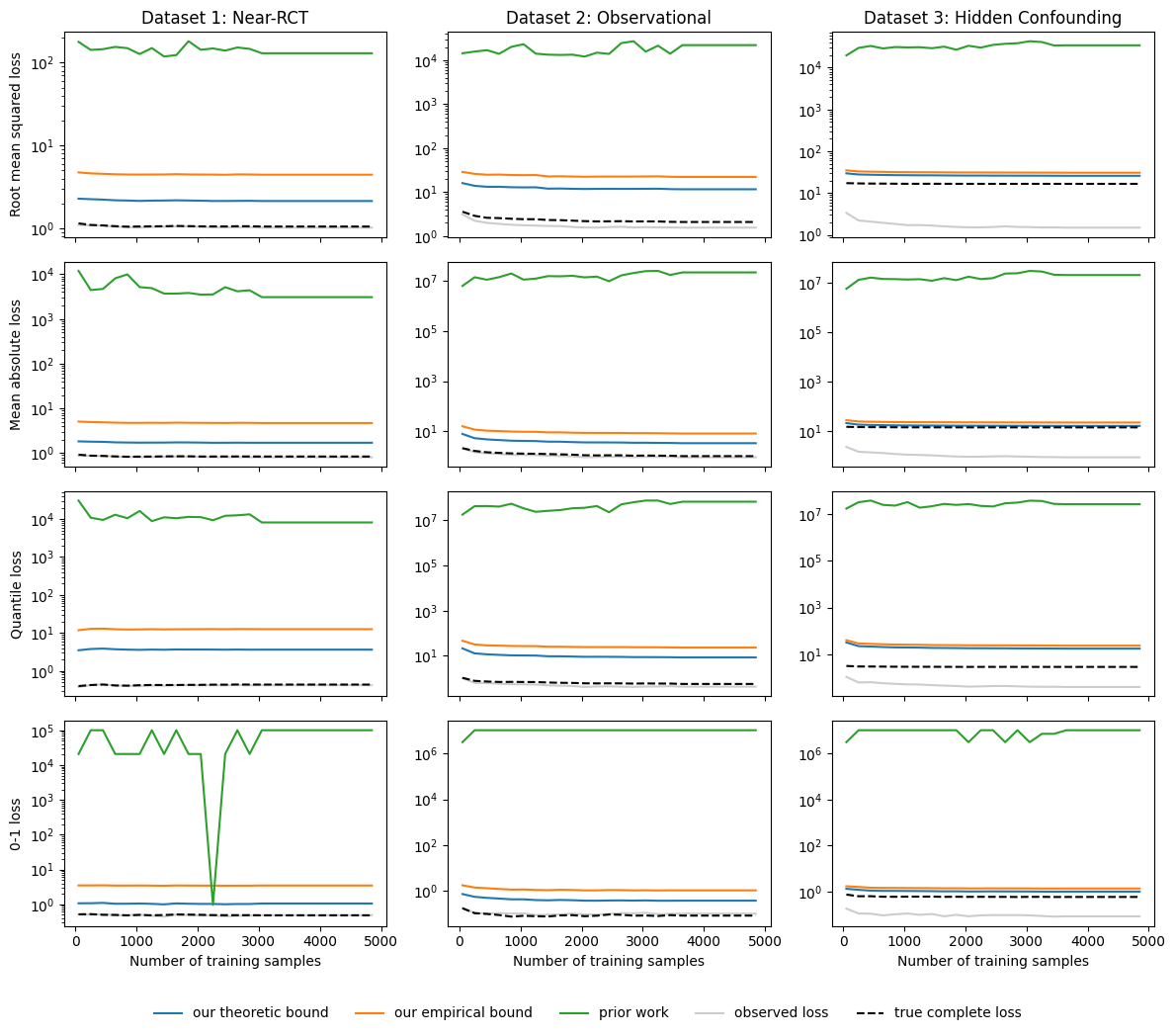}
    \caption{
        \textbf{Alternate version of Figure~\ref{fig:figure-1} for T-learners and including more losses.}
    }
\end{figure}

\begin{figure}[ht]
    \centering
    \subfloat[][Near-RCT dataset.]{\includegraphics[width=.3\textwidth]{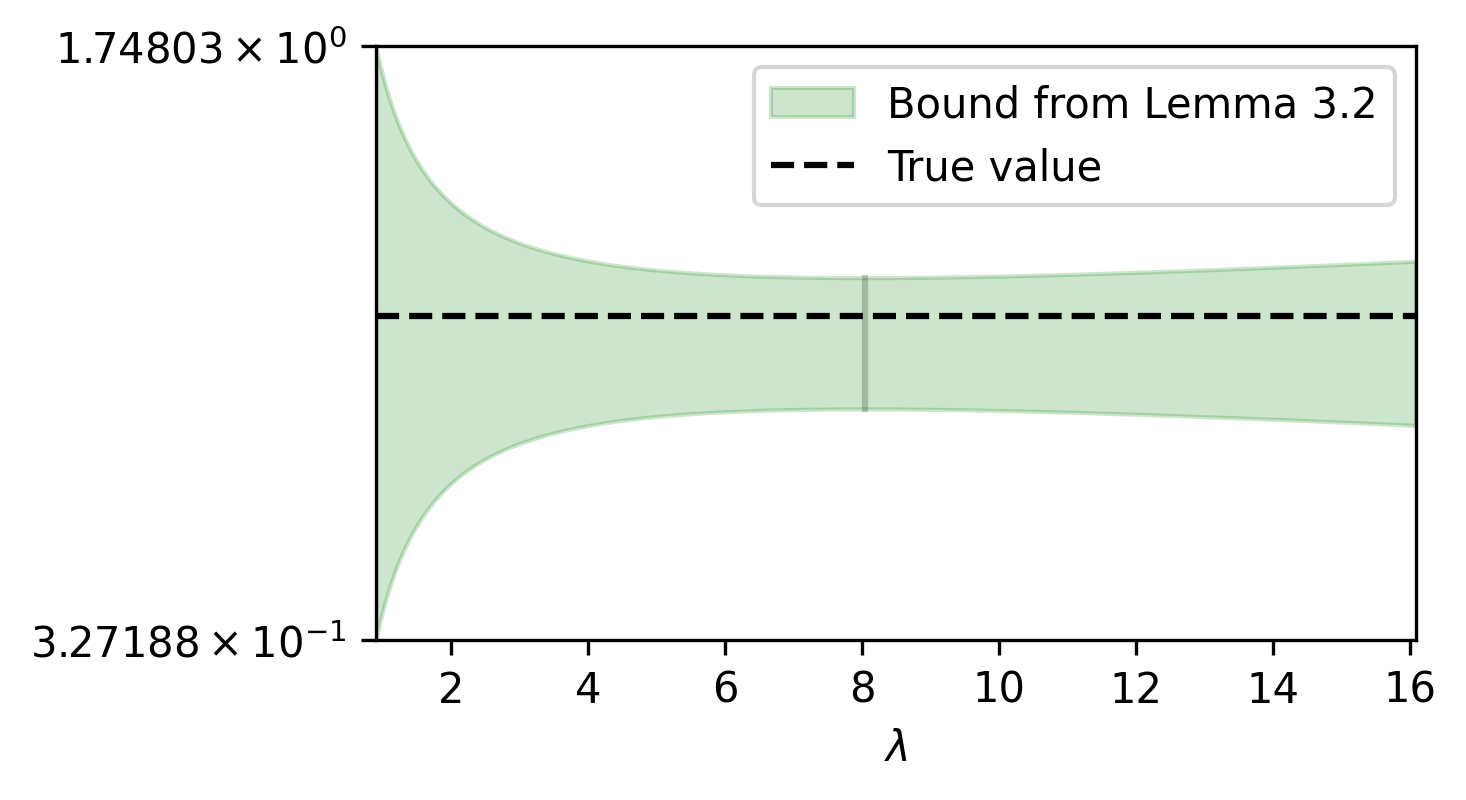}}
    \subfloat[][Observational dataset.]{\includegraphics[width=.3\textwidth]{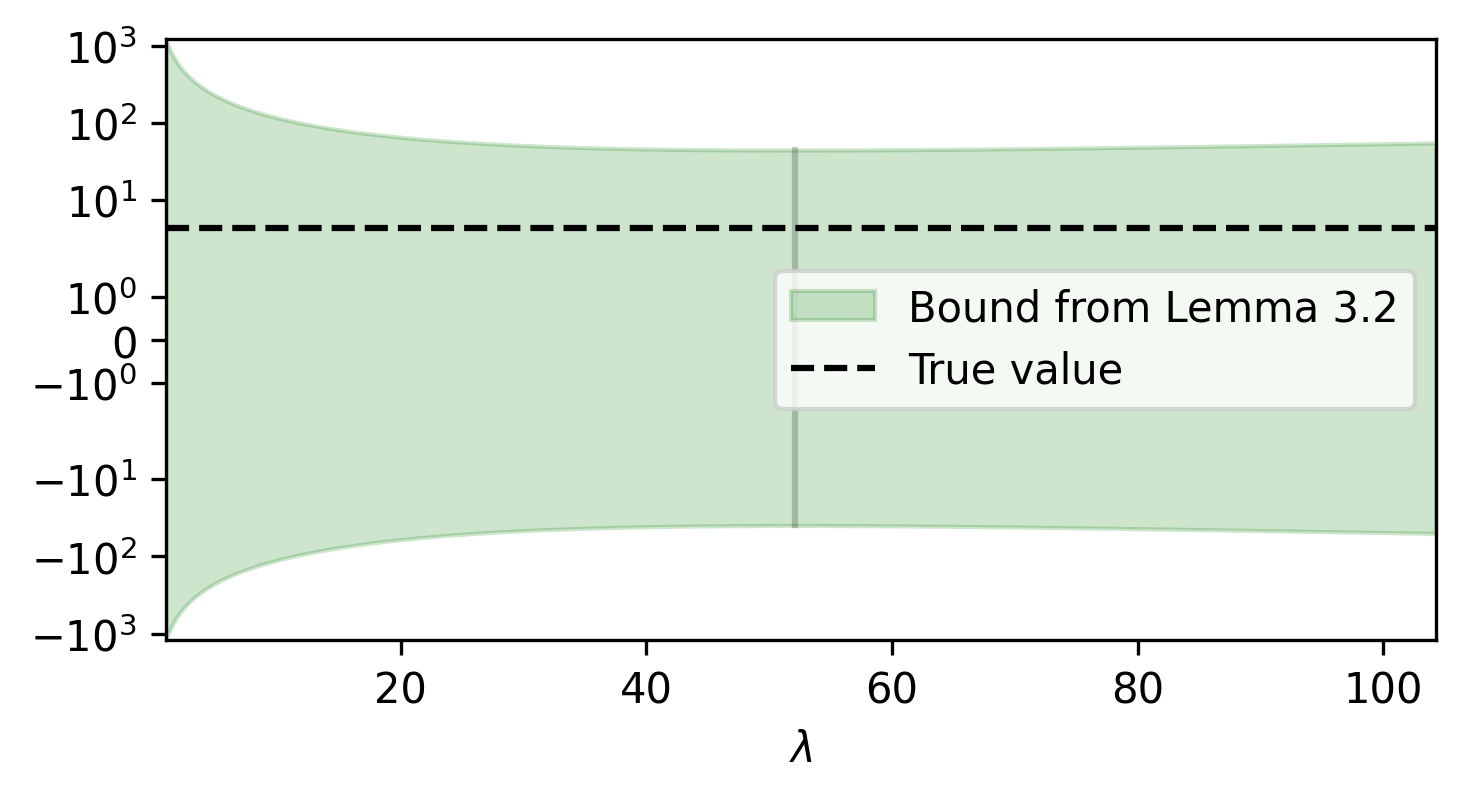}}
    \subfloat[][Hidden Confounding dataset.]{\includegraphics[width=.3\textwidth]{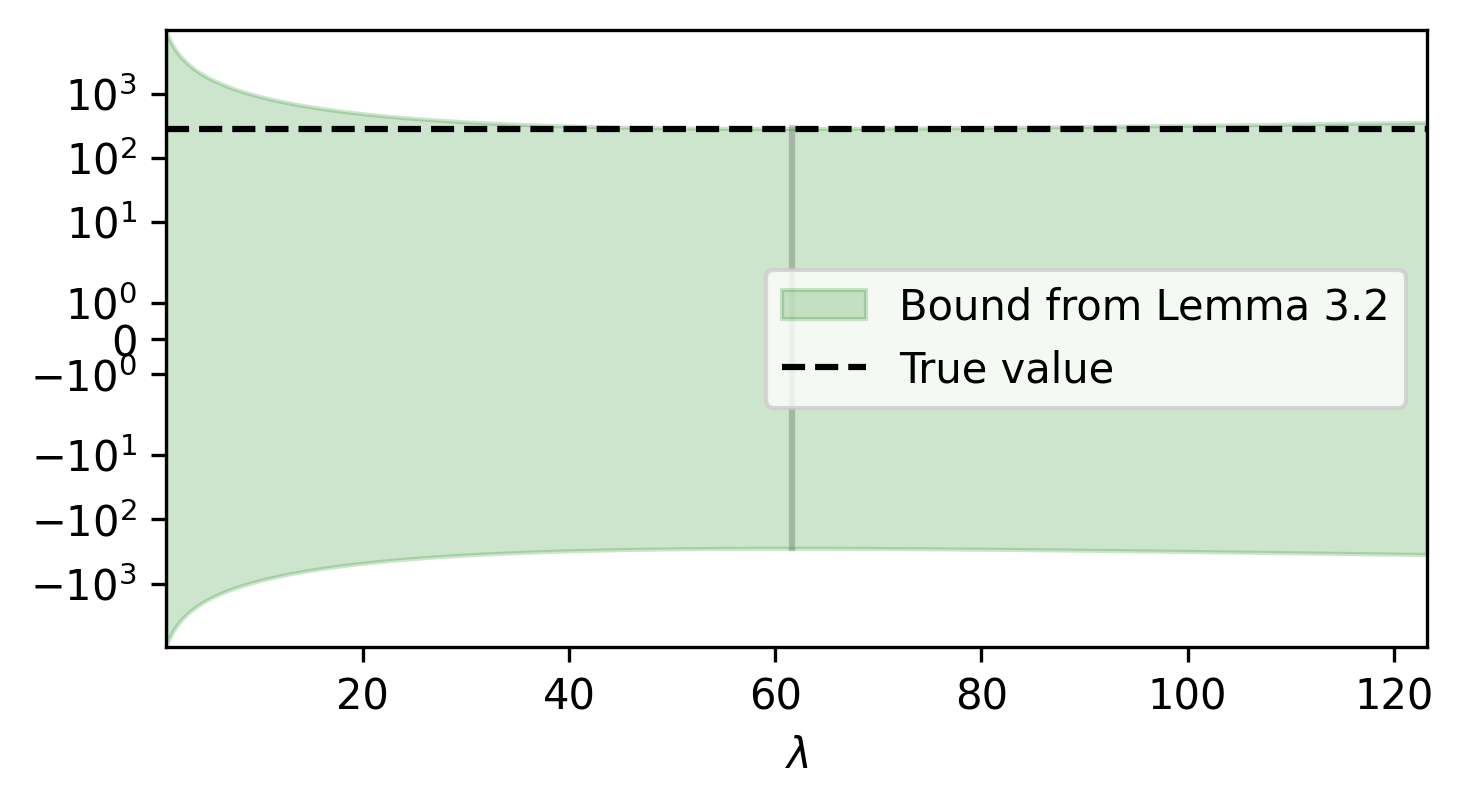}}
    \caption{
        \textbf{Alternate versions of Figure~\ref{fig:importance-of-tuning-parameter}.}
    }
\end{figure}

\begin{figure}[ht]
    \centering
    \includegraphics[width=.99\textwidth]{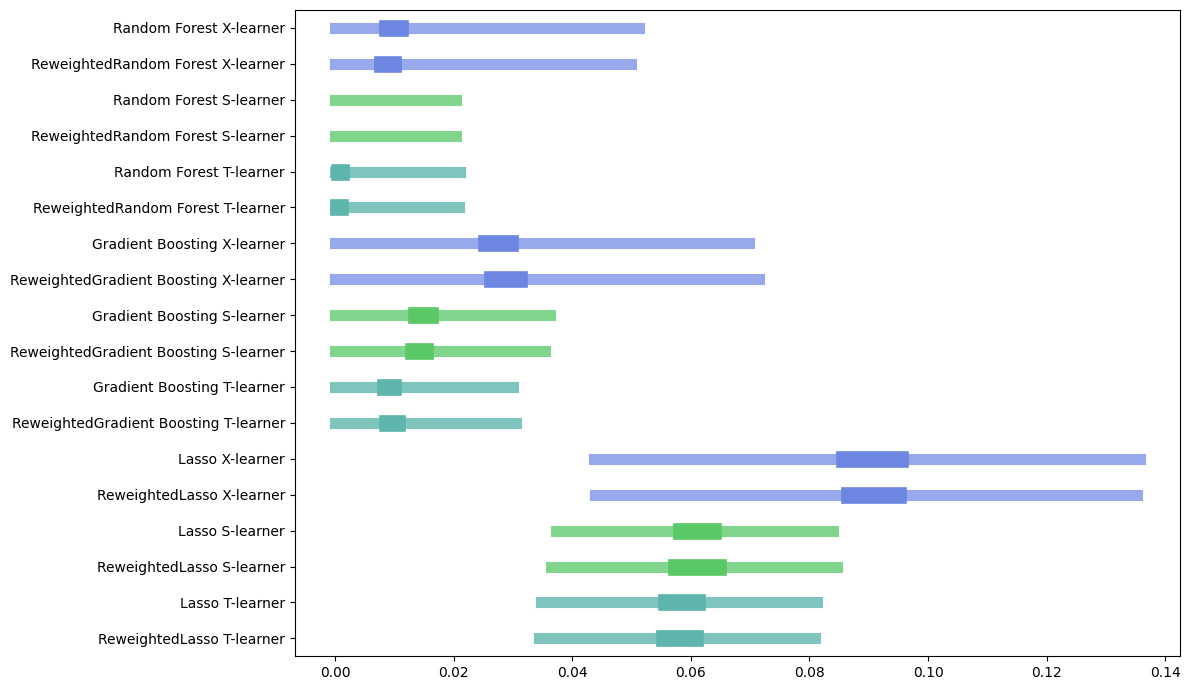}
    \caption{
        \textbf{Alternate version of Figure~\ref{fig:figure-3} which includes models using sample reweighting.}
    }
\end{figure}

\end{document}